\documentclass{article}

\usepackage{conference}

\title{From Sweep to Seam: Interleaved Cross-Block Post-Training Quantization}

\author{
  Achille Jacquemond \\
  Fujitsu Limited
  \And
  Yuma Ichikawa \\
  Fujitsu Limited, RIKEN Center for AIP
  \And
  Akira Sakai \\
  Fujitsu Limited, Tokai University
}

\begin{document}
\maketitle

\begin{abstract}
  Compressing large language models to two bits or fewer is increasingly feasible through block-wise
  post-training quantization; cross-block variants reconstruct neighboring Transformer blocks within a moving
  window. In the fixed two-block setting studied here, the matched sequential baseline moves this window
  through the network once, so errors introduced early in the sweep are not revisited. We propose Interleaved
  Cross-Block Quantization (ICBQ), a scheduling modification that revisits the boundary pair between
  consecutive chunks. Each seam pair is refined twice: first at the end of one chunk and again at the start of
  the next. The method retains the local two-block objective and reuses the calibration inputs of existing
  block-wise PTQ pipelines. Under stated local contraction and smoothness assumptions, we derive a depth-wise
  upper-bound comparison in which seam revisits multiply the propagated term while the residual remains
  bounded independently of depth. In the reported experiments, ICBQ reduces ternary-quantization perplexity
  relative to the matched Sequential CBQ baseline, yields finite perplexity in configurations where the
  baseline has severe degradation, and can also be used with 3-bit and 2-bit GPTQ.
\end{abstract}

\section{Introduction}
\label{sec:intro}

Deploying large language models (LLMs) at sub-two-bit precision is an important challenge in efficient ML
systems. In principle, one could directly optimize the quantized network end to end through quantization-aware
training (QAT), as in LLM-QAT and LittleBit~\citep{liu2024llm,lee2025littlebit}. Alternatively, one could
formulate a post-training quantization (PTQ) objective jointly over all blocks. At the scale of modern LLMs,
however, such global optimization is rarely feasible. Once models exceed even a few billion parameters,
storing the full computation graph, intermediate activations, and optimizer state on a single accelerator
typically exceeds the available memory budgets~\citep{ichikawa2026onecomp}. Consequently, practical PTQ
methods instead rely on layer- or block-local surrogate
objectives~\citep{frantar2022gptq,li2021brecq,shao2023omniquant,bulat2024qbb}. These methods aim to reduce
global quantization error by decomposing the problem into a sequence of local subproblems, whose peak memory
cost is primarily governed by the reconstruction unit. For large-scale LLMs, block-wise PTQ is a practical
route to reducing end-to-end output error within realistic hardware constraints.

Within practical PTQ pipelines, layer-wise weight-only methods such as GPTQ~\citep{frantar2022gptq} and
AWQ~\citep{lin2023awq} recover much of the floating-point accuracy at three- to four-bit weight precision,
while SmoothQuant~\citep{xiao2023smoothquant} targets W8A8 weight--activation PTQ by smoothing activation
outliers. Subsequent methods strengthen local objectives in complementary ways:
QEP~\citep{arai2025quantization} explicitly propagates and compensates quantization errors;
LoaQ~\citep{lin2025loaq} introduces output-matching factors within layer-wise PTQ; and
LPCD~\citep{ichikawa2025lpcd} optimizes relaxed objectives over arbitrary submodules before projecting them
with layer-wise quantizers. More recently, extreme low-bit methods such as OneBit~\citep{xu2024onebit},
BitNet~\citep{wang2023bitnet}, DBF~\citep{boza2026dbf}, and MDBF~\citep{ichikawa2025mdbf} have pushed
quantization into increasingly aggressive precision regimes.

CBQ~\citep{ding2025cbq} introduced cross-block reconstruction to capture dependencies across multiple
Transformer blocks. For a controlled schedule comparison, we adapt this principle to a fixed two-block,
weight-only reconstruction objective; we refer to the resulting left-to-right baseline as Sequential CBQ for
brevity.

Our Sequential CBQ baseline applies each adjacent two-block window once in a left-to-right sweep after all
blocks have been quantized. Residual errors introduced at a block pair are then propagated to subsequent
blocks without another refinement of that pair. This can increase final activation mismatch with depth. The
effect can be pronounced under our ternary DBF setting. For example, Sequential CBQ produces very large
perplexities for Llama-$3$-$8$B and Qwen$3$-$8$B in our setting, motivating a schedule that revisits selected
pairs after additional blocks have been processed.

A schedule-level way to mitigate this accumulation is to revisit selected block pairs. We focus on chunk
seams: the boundary pairs between consecutive chunks of newly quantized blocks. By interleaving a short CBQ
sweep with progressive quantization, the last pair of one chunk also becomes the first pair of the next, so
every seam is refined twice, as illustrated in Figure~\ref{fig:concept}. The revisits reuse the same
calibration inputs but add refinement computation. We formalize their effect as an additional contraction
factor on the propagated term of a depth-wise error bound. Thus, for a network of depth $L$ processed in
chunks of size $K$, the bound reflects the number of chunk seams. We evaluate this schedule with ternary DBF
and GPTQ inner quantizers.

\begin{figure}[tb]
  \centering
  \includegraphics[width=0.9\linewidth]{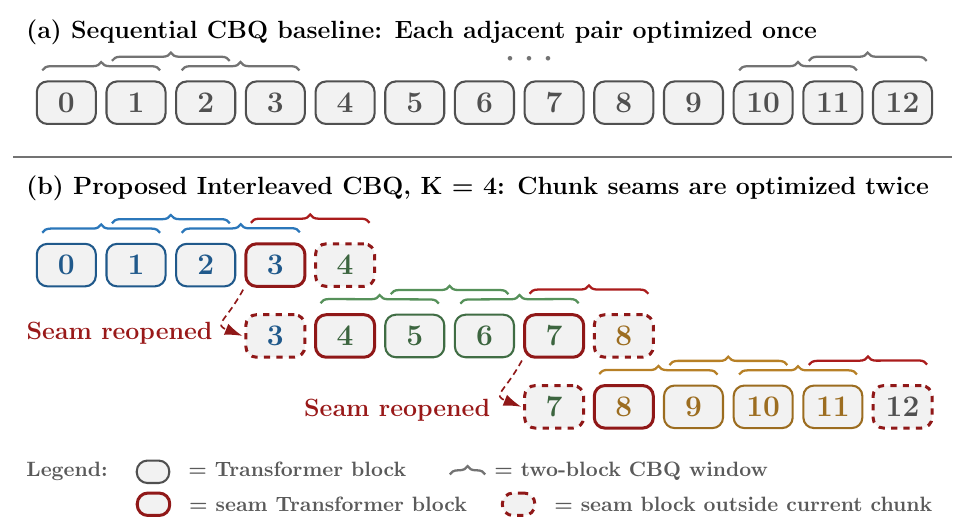}
  \caption{\textbf{ICBQ revisits chunk seams using the same calibration samples} ($L=13$, $K=4$).
    \textbf{Panel a. Sequential CBQ} performs a single post-hoc sweep, optimizing each adjacent block pair once. \textbf{Panel b. Interleaved CBQ} inserts a short CBQ pass after each chunk of newly quantized blocks. Because the final window of chunk~$c$ is also the first window of chunk~$c{+}1$, the seam pairs $(3,4)$, $(7,8)$, and $(11,12)$ are optimized twice, whereas non-seam pairs are optimized once. Between chunks, student and teacher activations are recomputed from the stored layer-$0$ inputs.}
  \label{fig:concept}
  \vspace{-5pt}
\end{figure}

\section{Preliminaries}
\label{sec:prelim}

\paragraph{Notation and inner quantizers.} We consider a Transformer consisting of $L$ blocks. The $b$-th fp16
teacher block is denoted by $f_b^\star$, where $b \in \{0,\dots,L{-}1\}$. Given a calibration set
$\mathcal{X} = \{x^{(n)}\}_{n=1}^{N}$, the teacher activations are defined by
\begin{equation}
  X_0^{\star,(n)} = x^{(n)}, ~~ X_{b+1}^{\star,(n)} = f_b^\star\bigl(X_b^{\star,(n)}\bigr).
\end{equation}
The corresponding quantized student block is written as $f_b^{\theta_b}$, and the student activations are
denoted by $\hat{X}_b^{(n)}$. For each sample, an activation tensor has shape $(T,d_{\mathrm{hidden}})$, where
$T$ is the sequence length and $d_{\mathrm{hidden}}$ is the hidden width. Thus,
$X_b^{\star,(n)},\hat{X}_b^{(n)} \in \mathbb{R}^{T\times d_{\mathrm{hidden}}}$. Unless otherwise stated, all
activation norms are Frobenius norms and are denoted by $\|\cdot\|_F$. We use $N$ exclusively for the number
of calibration samples, and reserve $S=S(L,K)$ for the number of chunk seams introduced in
Section~\ref{sec:method}.

An \emph{inner quantizer} is the low-bit routine used to initialize a single Transformer block before any
cross-block refinement. We denote it by $f_b^{\theta_b} = Q(f_b^\star; \mathcal{Z}_b)$, where
$\mathcal{Z}_b = \{Z_b^{(n)}\}_{n=1}^{N}$ is the set of calibration activations supplied to block $b$ during
quantization. The role of $Q$ is local: it receives one fp16 block and produces one low-bit student block
whose weights belong to a prescribed codebook. The block-wise and cross-block procedures below do not depend
on the specific choice of $Q$.

Representative choices for $Q$ include Hessian-aware quantizers, such as GPTQ~\citep{frantar2022gptq}, as well
as factorized quantizers designed for extreme low-bit settings, such as DBF~\citep{boza2026dbf} and
MDBF~\citep{ichikawa2025mdbf}. The original DBF represents a weight matrix using two binary sign matrices and
diagonal scaling matrices. In our experiments, we use a ternary extension in which the factor entries take
values in $\{-1,0,+1\}$; unless otherwise stated, DBF below refers to this ternary implementation.
Specifically, for a weight matrix $W\in\mathbb{R}^{d_{\mathrm{out}}\times d_{\mathrm{in}}}$, our
parameterization is
\begin{equation}
  \widehat{W} = D_a A D_m B D_b, ~~ D_a = \operatorname{diag}(a), ~ D_m = \operatorname{diag}(m), ~ D_b = \operatorname{diag}(b).
  \label{eq:ternary-dbf}
\end{equation}
where $A\in \{-1,0,+1\}^{d_{\mathrm{out}}\times k}$ and $B\in \{-1,0,+1\}^{k\times d_{\mathrm{in}}}$ are
ternary factors, $a\in\mathbb{R}_{\ge 0}^{d_{\mathrm{out}}}$, $m\in\mathbb{R}_{\ge 0}^{k}$, and
$b\in\mathbb{R}_{\ge 0}^{d_{\mathrm{in}}}$ are nonnegative scaling vectors, and $k$ denotes the intermediate
dimension that determines the compression ratio. We use this ternary DBF implementation as the primary inner
quantizer, while GPTQ serves as a complementary baseline for evaluating whether the proposed schedule
generalizes beyond a single quantization method.

\paragraph{Block-wise post-training quantization and prefit.} Post-training quantization compresses a
pretrained model using only a small calibration set, without retraining on the original training corpus. A
natural but impractical objective would be to quantize all Transformer blocks jointly so that the final output
of the student model matches that of the fp16 teacher model. For modern LLMs, this global objective is usually
too expensive because it requires storing activations and optimization state across many blocks. Block-wise
PTQ avoids this cost by decomposing the problem into local subproblems, quantizing one Transformer block at a
time \citep{li2021brecq,shao2023omniquant,frantar2022gptq}.

In a standard block-wise PTQ pipeline, the blocks are processed in depth order. At block $b$, the inner
quantizer $Q$ is applied to the fp16 block $f_b^\star$ using the calibration activations available at that
depth:
\begin{equation}
  f_b^{\theta_b} \leftarrow Q\bigl(f_b^\star;\mathcal{Z}_b\bigr).
\end{equation}
After block $b$ has been quantized, the student activation stream is advanced by
\begin{equation}
  \hat{X}_{b+1}^{(n)} = f_b^{\theta_b}\bigl(\hat{X}_b^{(n)}\bigr), ~~ n=1,\dots,N.
\end{equation}
Thus, later blocks are calibrated using activations that already include the effects of earlier quantized
blocks. This is the key distinction between block-wise PTQ and purely layer-isolated quantization: the method
still optimizes local blocks, but it propagates the student activations through the quantized prefix.

Before applying the inner quantizer, one may optionally perform a \emph{prefit} step. Prefit is a
high-precision reconstruction step that adjusts a floating-point copy of the block so that it better matches
the fp16 teacher block on the calibration inputs. Given the calibration inputs $\mathcal{Z}_b$ for block $b$,
the prefit solves
\begin{equation}
  \phi_b^{\mathrm{pre}} = \arg\min_{\phi_b} \frac{1}{N}\sum_{n=1}^{N} \left\|f_b^{\phi_b}\bigl(Z_b^{(n)}\bigr)-f_b^\star\bigl(Z_b^{(n)}\bigr)\right\|_F^2.
  \label{eq:prefit}
\end{equation}
The resulting high-precision block $f_b^{\phi_b^{\mathrm{pre}}}$ is then passed to the inner quantizer $Q$.
Importantly, prefit itself does not specify a low-bit codebook and is therefore independent of the choice of
$Q$; it only prepares a better floating-point starting point for the subsequent quantization step. In this
work, we enable prefit only when the inner quantizer is ternary DBF.

\paragraph{Cross-block quantization.} The limitation of the block-wise procedure above is that each block is
judged mainly by how well it reconstructs its own fp16 output. However, in a Transformer, the output error of
block $i$ becomes the input error of block $i{+}1$. A locally small error at the output of one block may
therefore be amplified or rotated by the next block, especially at very low precision. CBQ~\citep{ding2025cbq}
addresses this issue through cross-block reconstruction over multiple blocks. To isolate the scheduling effect
studied here, we use a simplified weight-only adaptation with a fixed window of two adjacent Transformer
blocks, which keeps the memory footprint local.

For the adjacent window $(i,i{+}1)$, our two-block CBQ-style objective minimizes the reconstruction error at
the output of the second block:
\begin{equation}
  \min_{\theta_i,\theta_{i+1}} \mathcal{L}_{\mathrm{cbq}}(\theta_i,\theta_{i+1}) = \frac{1}{N} \sum_{n=1}^{N} \left\|f_{i+1}^{\theta_{i+1}}\bigl(f_i^{\theta_i}(\hat{X}_i^{(n)})\bigr)-Y_{i+2}^{(n),\star}\right\|_F^2,
  \label{eq:cbq}
\end{equation}
where the fp16 two-block target is
\begin{equation}
  Y_{i+2}^{(n),\star} = f_{i+1}^{\star}\bigl(f_i^{\star}(X_i^{\star,(n)})\bigr).
\end{equation}
Here, $\hat{X}_i^{(n)}$ is the current student-side activation entering the window, while $X_i^{\star,(n)}$ is
the corresponding teacher activation. The objective therefore asks the current student window to reproduce the
fp16 teacher output after two consecutive blocks. In practice, we optimize Eq.~\eqref{eq:cbq} with a rollback
safeguard: if the updated window does not reduce the reconstruction MSE, the parameters are restored to their
previous values. The Sequential CBQ baseline used in this paper first initializes all blocks with the inner
quantizer $Q$, and then performs a single left-to-right sweep over the adjacent windows
$(0,1),(1,2),\dots,(L{-}2,L{-}1)$, solving Eq.~\eqref{eq:cbq} once for each window and advancing the student
activation stream after every step. Thus, throughout this paper, Sequential CBQ denotes this fixed two-block
adaptation: $Q$ provides the initial low-bit blocks, and the CBQ-style objective refines neighboring blocks
independently of the choice of inner quantizer.

\section{Interleaved Cross-Block Quantization}
\label{sec:method}

ICBQ keeps the progressive block-wise PTQ pipeline, but partitions the depth axis into consecutive chunks of
$K$ blocks and inserts one short interleaved CBQ sweep after each chunk. At each chunk boundary, the
refinement includes adjacent pairs around the boundary. For non-final chunks, the right block of the boundary
pair can still be unquantized and is handled as a provisional copy until its own quantization step. When the
next chunk starts, that same boundary pair is refined again, creating seam overlap that drives seam-amortized
contraction and suppresses depth-wise error buildup. After each chunk-level refinement, activation re-roll
recomputes student and teacher streams from the saved layer-$0$ inputs so the next quantization/refinement
step sees a coherent state.

\subsection{Outer driver}
\label{subsec:driver}

The outer driver is the orchestrator of ICBQ: it decides when each standard block-wise PTQ operation runs and
in what order. It sweeps blocks left-to-right ($b=0,\dots,L{-}1$), keeps a chunk-start pointer
$b_{\mathrm{cs}}$, and for each block applies the usual local quantization step (with optional prefit), then
advances the student and teacher streams by one layer through $f_b^Q$ and $f_b^\star$. This follows the
standard block-wise PTQ progression from Section~\ref{sec:prelim}, but adds explicit chunk-level control over
when interleaved cross-block refinement is triggered.

The chunk closes whenever $b + 1 - b_{\mathrm{cs}} \ge K$ or $b = L{-}1$, i.e., once the current chunk reaches
size $K$ (or the model ends). At closure, the orchestrator launches a short interleaved CBQ call
(Algorithm~\ref{alg:icbq} in Appendix~\ref{app:pseudo}) on the window range
\begin{equation}
  [w_{\min}, w_{\max}) = \bigl[ \max(0, b_{\mathrm{cs}}-1),\ \min(b{+}1, L{-}1) \bigr),
  \label{eq:window-range}
\end{equation}
solving~\eqref{eq:cbq} on each pair $(i,i{+}1)$ with $i=w_{\min},\dots,w_{\max}{-}1$ (the upper limit avoids
the non-existent pair $(L{-}1,L)$). After this local refinement, the driver re-rolls activations over the
processed prefix and advances to the next chunk; the exact inner-loop and re-roll mechanics are detailed in
Section~\ref{subsec:reroll}. It then sets $b_{\mathrm{cs}} \leftarrow b+1$ and continues to the next chunk. In
short, the outer driver preserves block-by-block PTQ while interleaving per-chunk CBQ: quantize progressively,
refine at chunk boundaries, re-roll, repeat.

\subsection{The seam: the structural heart of ICBQ}
\label{subsec:seam}

In block-wise PTQ, quality is decided at adjacent-block interfaces. This subsection isolates the key ICBQ
overlap effect at those interfaces: chunk seams are revisited once at the end of one chunk and again at the
start of the next. Writing $b_{\mathrm{cs}}=(c{-}1)K$ for chunks indexed by $c=1,2,\dots$ (with the final
chunk truncated at the last valid pair index), the window range~\eqref{eq:window-range} becomes
\begin{equation}
  \Bigl[ \max(0, (c{-}1)K{-}1),\ \min(cK, L{-}1) \Bigr),
  \label{eq:chunk-pair-range}
\end{equation}
that is, the pairs optimized in chunk~$c$ are
\begin{equation}
  (i,i{+}1), ~~ i=\max(0, (c{-}1)K{-}1),\ \dots,\ \min(cK, L{-}1)-1.
  \label{eq:pairs}
\end{equation}
For every non-final chunk boundary $cK<L$, the \emph{last} pair processed in chunk~$c$ is $(cK{-}1,cK)$. When
chunk~$c{+}1$ starts, its first valid pair index is again $cK{-}1$, so the \emph{first} pair processed is the
same boundary pair $(cK{-}1,cK)$. Therefore each seam pair is refined twice (end of chunk~$c$, then start of
chunk~$c{+}1$), while non-seam interior pairs are refined once (Figure~\ref{fig:concept}). We call
$(cK{-}1,cK)$ the \emph{$c$-th seam}, and the number of seams in an $L$-block model with chunk size $K$ is
\begin{equation}
  S(L,K) = \bigl\lceil L/K\bigr\rceil - 1.
  \label{eq:seams}
\end{equation}
This ``visited twice'' seam structure is stated formally as Lemma~\ref{lem:seam-structural} in
Appendix~\ref{app:proof}. Intuitively, seams are exactly where block-wise errors are handed from one chunk to
the next, so revisiting them is the core mechanism that suppresses depth-wise buildup. Theorem~\ref{thm:seam}
quantifies this effect through an extra seam-level contraction on the propagated term.

\subsection{Activation re-roll for coherent interleaved refinement}
\label{subsec:reroll}

The pair updates themselves follow the same regular CBQ subproblem from Section~\ref{sec:prelim}; in ICBQ,
they are interleaved after each chunk during ongoing block-wise quantization. Given the window range
$[w_{\min},w_{\max})$ from the outer driver, the inner call performs a local block-wise pair sweep. It first
recomputes the student and teacher streams to the window start (through blocks $0,\dots,w_{\min}{-}1$), then
iterates over pair indices $i=w_{\min},\dots,w_{\max}{-}1$: it recomputes the teacher target $Y_{i+2}^{\star}$
from $f_i^\star,f_{i+1}^\star$, updates $(\theta_i,\theta_{i+1})$ by minimizing the two-block
loss~\eqref{eq:cbq}, and advances the student/teacher streams one layer via $f_i^{\theta_i}$ and $f_i^\star$.
Pseudocode is given in Appendix~\ref{app:pseudo}.

Because this sweep is interleaved at chunk boundaries, each chunk-level refinement can update blocks whose
outputs are still needed by subsequent chunks, so cached activations can become stale. Re-roll resolves this
by recomputing coherent student and teacher streams from saved layer-$0$ inputs: locally to the window start
during the within-chunk sweep and through the processed prefix $0,\dots,b$ at chunk end before the next outer
step. This mechanism is encoded by Assumption~\ref{asm:reroll} and used by Theorem~\ref{thm:seam}.

\paragraph{Peak-memory footprint.} The only state added by ICBQ is a single constant-size buffer of depth-$0$
activations,
\begin{equation}
  \Delta\mathrm{mem}_{\mathrm{ICBQ}} = N\cdot T\cdot d_{\mathrm{hidden}}\cdot\mathrm{bytes}(\mathrm{dtype}),
  \label{eq:icbq-delta-mem}
\end{equation}
which depends only on the calibration workload---not on the model depth $L$ or the chunk size $K$. Since the
re-roll is executed after the previous window is released, ICBQ never holds two active windows concurrently,
and
\begin{equation}
  \mathrm{Peak}_{\mathrm{ICBQ}} \le \mathrm{Peak}_{\mathrm{CBQ}}+\Delta\mathrm{mem}_{\mathrm{ICBQ}},
  \label{eq:icbq-peak}
\end{equation}
an $L,K$-independent additive constant. At $T = 2048$ in fp16 the buffer is $\approx 4.3$ GB with $N = 256$
and $d_{\mathrm{hidden}} = 4096$, below $15\%$ of the ternary DBF peak reported in Appendix~\ref{app:hyper}.
Thus, ICBQ has no depth- or chunk-dependent memory growth beyond the added buffer and can be integrated into
CBQ-style progressive PTQ pipelines.

\section{Theory: Seam-amortized Error Contraction}
\label{sec:theory}

This section explains the mechanism behind why interleaving with seam revisits can be more stable than plain
sequential CBQ. The interpretation is intentionally worst-case: it compares upper bounds, not guaranteed
realized perplexity on every run. Full assumptions and proofs are in Appendix~\ref{app:assumptions}. The
tracked quantity is the boundary mismatch passed from one block pair to the next:
$m_{i}\coloneqq\|\hat{X}_{i+1}-X_{i+1}^{\star}\|_{F}$. The reader can interpret this as ``how much activation
error is handed from pair $(i{-}1,i)$ to pair $(i,i{+}1)$.'' We denote by $m_{i}^{\mathrm{out}}$ the mismatch
after the CBQ call on $(i,i{+}1)$, with $m_{-1}^{\mathrm{out}}=0$ from the shared start
$\hat{X}_{0}=X_{0}^{\star}$.

\paragraph{Base behavior away from seams.} Under Assumptions~\ref{asm:lip}--\ref{asm:reroll}
(Appendix~\ref{app:assumptions}), each regular two-block CBQ step has a damp-and-inject form: it contracts
part of incoming mismatch (through $\tau\coloneqq\gamma\rho<1$) and adds local approximation noise
($\epsilon_{\mathrm{sub}}$). Formally, Lemma~\ref{lem:per-pair} (Appendix~\ref{app:per-pair}) gives
\begin{equation}
  m_{i}^{\mathrm{out}} \le \tau m_{i-1}^{\mathrm{out}} + (1{+}\gamma) \epsilon_{\mathrm{sub}}, ~~ m_{-1}^{\mathrm{out}} \coloneqq 0,
  \label{eq:per-pair-rec-main}
\end{equation}
for every non-seam pair in either schedule.

\paragraph{Why seams are special.} At a seam $(cK{-}1,cK)$, re-rolling resets the local context from stored
inputs, and that seam pair is effectively revisited across adjacent chunks. Intuitively, non-seam pairs get
one cleanup pass, while seams get an extra cleanup pass on carried mismatch. This yields
\begin{equation}
  m_{i}^{\mathrm{out}} \le \gamma\tau m_{i-1}^{\mathrm{out}} +(1{+}\gamma{+}\gamma^{2}) \epsilon_{\mathrm{sub}} ~~\text{at a seam}
  \label{eq:seam-rec-main}
\end{equation}
(Lemma~\ref{lem:seam}, Appendix~\ref{app:single-seam}). Chaining Eq.~\eqref{eq:per-pair-rec-main} with one
seam recurrence per seam across the depth gives the theorem below: each seam contributes one additional
multiplicative contraction factor in the bound.

\begin{theorem}[Seam-amortized error contraction; informal, formal:
  Theorem~\ref{thm:seam-formal} in Appendix~\ref{app:thm-formal}, proved in Appendix~\ref{app:proof}]
  \label{thm:seam}
  Let $e_{L}^{\mathrm{seq}}$ and $e_{L}^{\mathrm{int}}$ denote the depth-$L$ activation mismatches of
  Sequential CBQ and one-pass ICBQ with chunk size $K$, on the same calibration set, the same inner quantizer
  $Q$, and the same CBQ optimizer budget. Write $\mathcal{B}_{L}^{\mathrm{seq}}$ for the depth-$L$ upper bound
  obtained by unrolling Eq.~\eqref{eq:per-pair-rec-main} alone (Appendix~\ref{app:finale}). Under
  Assumptions~\ref{asm:lip}--\ref{asm:reroll} and $\tau<1$, for every $L\ge 2$ and $1\le K\le L$,
  \begin{equation}
    \mathbb{E}\bigl\|e_{L}^{\mathrm{int}}\bigr\|_{F} \le \gamma^{S(L,K)} \mathcal{B}_{L}^{\mathrm{seq}}+C(L,K,\rho,\gamma) \epsilon_{\mathrm{sub}}, ~~ S(L,K) = \lceil L/K\rceil-1,
    \label{eq:thm-seam}
  \end{equation}
  where $C(L,K,\rho,\gamma)\le 1+(1{+}\gamma{+}\gamma^{2})\rho/(1-\tau)$ \emph{uniformly in} $L,K$, and the
  per-seam factor $\gamma$ is tight at the recurrence level (Proposition~\ref{prop:tightness}).
\end{theorem}

\begin{corollary}[Geometric decay in the number of seams; informal, formal:
  Corollary~\ref{cor:geom-formal} in Appendix~\ref{app:thm-formal}]
  \label{cor:geom}
  Define the \emph{amortization gain}
  $G(L,K)\coloneqq\mathcal{B}_{L}^{\mathrm{seq}}/\mathcal{B}_{L}^{\mathrm{int}}$, where
  $\mathcal{B}_{L}^{\mathrm{int}}$ is the right-hand side of Eq.~\eqref{eq:thm-seam}. If for some
  $\delta\ge 0$ the residual term satisfies
  $C \epsilon_{\mathrm{sub}}\le\delta \gamma^{S(L,K)}\mathcal{B}_{L}^{\mathrm{seq}}$, then
  $G(L,K)\ge\gamma^{-S(L,K)}/(1{+}\delta)$. Thus, when additive noise is not dominant, per-seam shrink factors
  compound geometrically: the ICBQ-over-sequential bound ratio is exponential in the number of seams.
\end{corollary}

Putting Theorem~\ref{thm:seam} and Corollary~\ref{cor:geom} together gives a simple picture: away from seams,
mismatch is damped but new local noise is injected; at seams, ICBQ gets an extra damping action, and these
actions compound with depth. Thus, chunk size controls how often this correction mechanism is triggered:
smaller $K$ means more seams and more opportunities for bound-level contraction using the same calibration
data. Equation~\eqref{eq:thm-seam} is still a bound-vs-bound comparison against the sequential
\emph{upper bound} $\mathcal{B}_{L}^{\mathrm{seq}}$, not necessarily realized sequential error
(Remark~\ref{rem:ub-semantics}). For $L=32$ and $K=4$ (so $S=7$), the measured aggregate
$\bar\gamma\approx 0.67$ on Llama-$2$-$7$B (Appendix~\ref{app:gamma-measure}) gives
$\bar\gamma^{-7}\approx 16\times$, while a conservative aggregate maximum $\gamma=0.92$ gives
$\approx 1.8\times$; the empirical range $1.08$--$14.5\times$ in Section~\ref{sec:exp} lies between these
diagnostics.

\section{Experiments}\label{sec:exp}
Our experiments assess whether ICBQ improves perplexity relative to Sequential CBQ and whether the schedule
also applies to a second weight-only inner quantizer. We then isolate chunk-size and prefit effects on a
shared model subset (Section~\ref{sec:ablation}, Tables~\ref{tab:chunk}--\ref{tab:prefit}).
\noindent\textbf{Formatting note.} In all tables of this section, \textbf{\textcolor{myred}{bold}} marks the
strongest value in each comparison group.

\subsection{Setup}\label{subsec:setup}

\paragraph{Experimental setup.} We evaluate seven publicly released base models spanning $3$B to $14$B
parameters. Every experiment is \emph{weight-only}; activations remain in fp16, and no rotation or activation
quantizer is applied. We evaluate ternary DBF ($1.58$-bit) with prefit enabled, and GPTQ at W3 (three bits,
per channel) and W2g128 (two bits, group size $128$) with prefit disabled. The chunk size is $K = 4$ unless
otherwise stated. We report C4-calibrated results in the main text, as C4 calibration was more stable than
WikiText-$2$ calibration for some non-standard-attention architectures in our setting
(Appendix~\ref{app:calibration-bias}). Our direct benchmark is the matched Sequential CBQ baseline defined
above, based on the fixed two-block adaptation of CBQ~\citep{ding2025cbq}; we additionally report
No-Refinement (inner quantizer only). During CBQ/ICBQ refinement, ternary DBF runs update the scaling vectors
and ternary factors (the factors through STE), while GPTQ runs update scales/zero-points and integer weights
(integer weights through STE). Additional setup details are listed in Appendix~\ref{app:setup-details}.

\paragraph{Metrics.} We report perplexity~(PPL,~$\downarrow$) on C4 test splits, using a standard
autoregressive language-model evaluation pipeline. We also report average zero-shot accuracy~($\uparrow$) over
7 standard benchmarks. Benchmark details and full per-task zero-shot values are provided in
Appendix~\ref{app:setup-details} and Appendix~\ref{app:zeroshot}.

\subsection{ICBQ improves perplexity over Sequential CBQ}\label{sec:exp-main}

\begin{table}[t]
  \centering
  \caption{\textbf{Ternary DBF results on seven base models.} ICBQ has lower PPL than Sequential CBQ in every reported PPL comparison and higher or equal average zero-shot accuracy in six of seven cases. FP model values are included for reference.}
  \label{tab:main}
  \small
  \setlength{\tabcolsep}{2.2pt}
  \renewcommand{\arraystretch}{1.03}
  \resizebox{\linewidth}{!}{%
    \begin{tabular}{@{}l*{3}{c}@{\hspace{2pt}}|@{\hspace{2pt}}*{9}{c}@{}}
      \toprule
      & \multicolumn{3}{c@{\hspace{2pt}}|@{\hspace{2pt}}}{FP model} & \multicolumn{3}{c}{No-Ref. (ternary DBF only)} & \multicolumn{3}{c}{Seq.\ CBQ} & \multicolumn{3}{c}{\textbf{ICBQ (ours)}} \\
      \cmidrule(lr){2-4} \cmidrule(lr){5-7} \cmidrule(lr){8-10} \cmidrule(lr){11-13}
  Metric & \begin{tabular}[c]{@{}c@{}}PPL\\(Wiki-$2$)\end{tabular} & \begin{tabular}[c]{@{}c@{}}PPL\\(C4)\end{tabular} & Avg.\ ZS & \begin{tabular}[c]{@{}c@{}}PPL\\(Wiki-$2$)\end{tabular} & \begin{tabular}[c]{@{}c@{}}PPL\\(C4)\end{tabular} & Avg.\ ZS & \begin{tabular}[c]{@{}c@{}}PPL\\(Wiki-$2$)\end{tabular} & \begin{tabular}[c]{@{}c@{}}PPL\\(C4)\end{tabular} & Avg.\ ZS & \begin{tabular}[c]{@{}c@{}}PPL\\(Wiki-$2$)\end{tabular} & \begin{tabular}[c]{@{}c@{}}PPL\\(C4)\end{tabular} & Avg.\ ZS \\
      \midrule
      \multicolumn{13}{@{}l}{\emph{Standard models}}\\
      Mistral-$7$B       & $4.69$ & $7.99$ & $0.7128$ & $276.31$ & $143.22$ & $0.3684$ & $36.5$ & $31.1$ & $0.4586$ & \textbf{\textcolor{myred}{13.8}} & \textbf{\textcolor{myred}{16.8}} & \textbf{\textcolor{myred}{0.5404}} \\
      Llama-$2$-$7$B     & $4.86$ & $7.27$ & $0.6668$ & $26.6$ & $17.71$ & $0.4856$ & $10.7$ & $13.6$ & $0.5330$ & \textbf{\textcolor{myred}{9.37}} & \textbf{\textcolor{myred}{12.5}} & \textbf{\textcolor{myred}{0.5533}} \\
      Llama-$2$-$13$B    & $4.35$ & $6.73$ & $0.6924$ & $9.92$ & $11.89$ & $0.5960$ & $7.94$ & $10.48$ & \textbf{\textcolor{myred}{0.6503}} & \textbf{\textcolor{myred}{7.36}} & \textbf{\textcolor{myred}{10.09}} & $0.6501$ \\
      \midrule
      \multicolumn{13}{@{}l}{\emph{Deep and non-standard-attention models}}\\
      Llama-$3.2$-$3$B    & $9.62$ & $15.63$ & $0.6320$ & $66.94$ & $61.44$ & $0.4482$ & $42.6$ & $40.63$ & $0.4706$ & \textbf{\textcolor{myred}{37.6}} & \textbf{\textcolor{myred}{37.01}} & \textbf{\textcolor{myred}{0.4857}} \\
      Llama-$3$-$8$B      & $7.37$ & $12.24$ & $0.6977$ & $154.43$ & $73.82$ & $0.4715$ & $42.95$ & $48.98$ & $0.5026$ & \textbf{\textcolor{myred}{34.97}} & \textbf{\textcolor{myred}{42.32}} & \textbf{\textcolor{myred}{0.5215}} \\
      Qwen$3$-$8$B        & $8.58$ & $13.93$ & $0.6946$ & $>1\mathrm{e}4$ & $>1\mathrm{e}4$ & $0.3495$ & $>1\mathrm{e}4$ & $5752.9$ & $0.4032$ & \textbf{\textcolor{myred}{25.86}} & \textbf{\textcolor{myred}{29.9}} & \textbf{\textcolor{myred}{0.5513}} \\
      Qwen$3$-$14$B       & $7.58$ & $12.48$ & $0.7292$ & $907.71$ & $484.59$ & $0.3529$ & $61.9$ & $44.82$ & \textbf{\textcolor{myred}{0.6871}} & \textbf{\textcolor{myred}{19.29}} & \textbf{\textcolor{myred}{24.67}} & \textbf{\textcolor{myred}{0.6871}} \\
      \bottomrule
    \end{tabular}%
  }
\end{table}
Table~\ref{tab:main} shows lower PPL for ICBQ than for Sequential CBQ in every reported comparison across
standard and deep/non-standard models. Qwen$3$-$8$B illustrates the difference: ICBQ yields finite PPL where
Sequential CBQ has divergence or very large PPL. The C4 average zero-shot metric follows the same overall
pattern; the exception is Llama-$2$-$13$B, where Sequential CBQ exceeds ICBQ by $0.0002$. Both schedules share
the stability condition $\gamma\rho<1$ (Proposition~\ref{prop:stability}), while Theorem~\ref{thm:seam} gives
ICBQ an additional $\gamma^{S(L,K)}$ factor on the propagated term when the stated assumptions hold.

\subsection{Schedule transfer to GPTQ \texorpdfstring{W3 and W2g128}{W3 and W2g128}}\label{sec:integer}
Table~\ref{tab:integer} shows that ICBQ attains the best or tied-best PPL on five of seven models at both W3
and W2g128, with one tie at each precision. The remaining cases are split between Sequential CBQ and
GPTQ-only. This variation is compatible with the bound view: when the propagated term is less dominant,
residual and optimization terms can determine the local ordering. Zero-shot results in
Appendix~\ref{app:zeroshot}, Table~\ref{tab:zeroshot-gptq}, show a similar aggregate pattern, alongside
task-specific exceptions.

\begin{table*}[t]
  \centering
  \caption{\textbf{The contraction transfers to GPTQ on the seven-model set.} PPL on C4 ($\downarrow$), weight-only GPTQ inner quantizer, fp16 activations, and no prefit ($T_{\mathrm{pre}} = 0$). Within each bit-width block, lower is better.}
  \label{tab:integer}
  \small
  \setlength{\tabcolsep}{4pt}
  \renewcommand{\arraystretch}{1.03}
  \begin{tabular}{@{}lcccccc@{}}
    \toprule
    & \multicolumn{3}{c}{W3 (per-channel)} & \multicolumn{3}{c}{W2g128 (group size $128$)} \\
    \cmidrule(lr){2-4} \cmidrule(lr){5-7}
    Model & GPTQ only & Seq.\ CBQ & ICBQ (ours) & GPTQ only & Seq.\ CBQ & ICBQ (ours) \\
    \midrule
    Mistral-$7$B     & $9.04$ & \textbf{\textcolor{myred}{8.77}} & $8.86$ & $13.29$ & \textbf{\textcolor{myred}{9.74}} & $9.76$ \\
    Llama-$2$-$7$B   & $8.50$ & $8.25$ & \textbf{\textcolor{myred}{8.20}} & $12.26$ & $9.28$ & \textbf{\textcolor{myred}{9.22}} \\
    Llama-$2$-$13$B  & $7.52$ & \textbf{\textcolor{myred}{7.39}} & \textbf{\textcolor{myred}{7.39}} & $10.00$ & \textbf{\textcolor{myred}{8.64}} & $8.76$ \\
    Llama-$3.2$-$3$B & $25.58$ & $18.63$ & \textbf{\textcolor{myred}{18.18}} & $58.16$ & $24.03$ & \textbf{\textcolor{myred}{22.37}} \\
    Llama-$3$-$8$B   & $31.42$ & $15.29$ & \textbf{\textcolor{myred}{14.63}} & $60.65$ & $18.51$ & \textbf{\textcolor{myred}{17.37}} \\
    Qwen$3$-$8$B     & $15.65$ & $15.96$ & \textbf{\textcolor{myred}{15.61}} & $23.08$ & $17.41$ & \textbf{\textcolor{myred}{17.01}} \\
    Qwen$3$-$14$B    & \textbf{\textcolor{myred}{13.54}} & $13.80$ & $13.63$ & $16.21$ & $14.74$ & \textbf{\textcolor{myred}{14.61}} \\
    \bottomrule
  \end{tabular}
\end{table*}

\subsection{The seam is the story: schedule ablations}\label{sec:ablation}
Returning to ternary DBF, we isolate two schedule-level factors on the same seven-model set used throughout
the main experiments: chunk size $K$ and progressive prefit.

\paragraph{Effect of chunk size.} Chunk size is the schedule's primary control knob: it sets how many seam
revisits ICBQ performs. Table~\ref{tab:chunk} reports PPL on C4 while varying $K \in \{2,4,8,L\}$.
Theorem~\ref{thm:seam} predicts stronger contraction as $K$ shrinks because smaller chunks create more seams.
Empirically, PPL improves monotonically or near-monotonically as $K$ decreases.

\begin{table}[t]
  \centering
  \caption{\textbf{Effect of chunk size $K$ on the seven-model set} (PPL on C4 ($\downarrow$), ternary DBF). $K=L$ is the no-interleaving sequential baseline.}
  \label{tab:chunk}
  \small
  \setlength{\tabcolsep}{6pt}
  \begin{tabular}{@{}lcccc@{}}
    \toprule
    Base model & $K=L$ (Seq.) & $K=8$ & $K=4$ & $K=2$ \\
    \midrule
    Mistral-$7$B      & $31.1$ & $17.33$ & $16.8$ & \textbf{\textcolor{myred}{14.85}} \\
    Llama-$2$-$7$B    & $13.6$ & $13.07$ & $12.5$ & \textbf{\textcolor{myred}{12.43}} \\
    Llama-$2$-$13$B   & $10.48$ & $10.26$ & $10.09$ & \textbf{\textcolor{myred}{10.03}} \\
    Llama-$3.2$-$3$B  & $40.63$ & $38.23$ & $37.01$ & \textbf{\textcolor{myred}{35.62}} \\
    Llama-$3$-$8$B    & $48.98$ & $45.26$ & $42.32$ & \textbf{\textcolor{myred}{39.61}} \\
    Qwen$3$-$8$B      & $5752.9$ & $34.01$ & $29.90$ & \textbf{\textcolor{myred}{29.80}} \\
    Qwen$3$-$14$B     & $44.82$ & $29.02$ & \textbf{\textcolor{myred}{24.67}} & $25.89$ \\
    \bottomrule
  \end{tabular}
\end{table}

\paragraph{Effect of prefit.} Table~\ref{tab:prefit} isolates the prefit step on the same seven-model set (PPL
on C4) by setting $T_{\mathrm{pre}} = 0$ under both schedules. Because prefit can improve each block's
starting point, this comparison tests whether the observed differences are confined to the prefit-enabled
setting. Interleaving remains better than Sequential CBQ in the reported no-prefit runs, indicating that the
schedule effect is not limited to prefit-enabled quantization.

\begin{table}[t]
  \centering
  \caption{\textbf{Interleaving helps even without prefit.} PPL on C4 ($\downarrow$) with ternary DBF, with and without the prefit step ($T_{\mathrm{pre}} = 50$ vs.\ $0$), on the same seven-model set.}
  \label{tab:prefit}
  \footnotesize
  \setlength{\tabcolsep}{5pt}
  \resizebox{\linewidth}{!}{%
    \begin{tabular}{@{}lcccc@{}}
      \toprule
      Base model & Prefit + Seq.\ CBQ & Prefit + Interleaved CBQ (ours) & No-prefit + Seq.\ CBQ & No-prefit + Interleaved CBQ \\
      \midrule
      Mistral-$7$B     & $31.1$ & \textbf{\textcolor{myred}{16.8}} & $19.71$ & \textbf{\textcolor{myred}{18.9}} \\
      Llama-$2$-$7$B   & $13.6$ & \textbf{\textcolor{myred}{12.5}} & $15.50$ & \textbf{\textcolor{myred}{15.4}} \\
      Llama-$2$-$13$B  & $10.48$ & \textbf{\textcolor{myred}{10.09}} & $11.85$ & \textbf{\textcolor{myred}{11.80}} \\
      Llama-$3.2$-$3$B & $40.63$ & \textbf{\textcolor{myred}{37.01}} & $51.66$ & \textbf{\textcolor{myred}{50.00}} \\
      Llama-$3$-$8$B   & $48.98$ & \textbf{\textcolor{myred}{42.32}} & $68.53$ & \textbf{\textcolor{myred}{66.35}} \\
      Qwen$3$-$8$B     & $5752.9$ & \textbf{\textcolor{myred}{29.90}} & $47.28$ & \textbf{\textcolor{myred}{41.59}} \\
      Qwen$3$-$14$B    & $44.82$ & \textbf{\textcolor{myred}{24.67}} & $35.30$ & \textbf{\textcolor{myred}{31.42}} \\
      \bottomrule
    \end{tabular}
  }
\end{table}

\section{Related Work}\label{sec:related}
We organize prior PTQ work by the axis each method \emph{refines}---block unit, objective, weight code, or
basis---and position ICBQ as a \emph{schedule}-level refinement under a fixed pair-level objective.

\paragraph{Refining the block unit and objective.} BRECQ~\citep{li2021brecq},
OmniQuant~\citep{shao2023omniquant}, and QLLM~\citep{liu2023qllm} replace layer-wise with block-wise
reconstruction. CBQ~\citep{ding2025cbq} uses cross-block reconstruction to capture dependencies across
multiple blocks; our direct baseline instantiates this principle with fixed adjacent two-block windows. A
complementary thread sharpens the local \emph{objective} or accounts for cross-layer effects:
QEP~\citep{arai2025quantization} explicitly propagates quantization errors to compensate for their
accumulation; LoaQ~\citep{lin2025loaq} introduces output-matching factors to improve output-level
approximation within layer-wise PTQ; and LPCD~\citep{ichikawa2025lpcd} optimizes relaxed objectives over
arbitrary submodules and projects the solutions with layer-wise quantizers. ICBQ changes the \emph{schedule}
while retaining a fixed pair-level objective; it may be combined with compatible inner objectives.

\paragraph{Refining the weight code and basis.} BitNet~\citep{wang2023bitnet}, the $1.58$-bit line of
\citet{ma2024era}, DBF~\citep{boza2026dbf}, MDBF~\citep{ichikawa2025mdbf}, and OneBit~\citep{xu2024onebit}
sharpen binary or ternary weight representations; ICBQ is agnostic to this axis and uses the selected method
as its inner quantizer. QuIP~\citep{chee2024quip} applies incoherence processing for low-bit parameter
quantization. Rotation-based methods such as QuaRot~\citep{ashkboos2024quarot} and
SpinQuant~\citep{liu2024spinquant} facilitate low-bit weight--activation inference by reducing outlier
effects; these transformations target the quantization basis and may be combined with ICBQ as preprocessing.

\paragraph{Schedule theory beyond quantization.} The \emph{order} in which subproblems are solved controls
propagated error---a theme in asynchronous block-coordinate descent, multigrid V/W-cycles, and parallel
quasi-quantum annealing~\citep{ichikawa2025qqa}. ICBQ applies this perspective to block-wise PTQ by using the
chunk seam as a point for an additional refinement with the same calibration samples.

\section{Conclusion}\label{sec:conclusion}

Block-wise PTQ has mainly focused on what happens \emph{inside} the reconstruction window; ICBQ also considers
\emph{when} the window is applied. Interleaving pair-refinement passes with progressive quantization makes
each chunk seam a second refinement point and adds a factor of $\gamma$ to the propagated term of the
depth-wise error upper bound under the stated assumptions (Theorem~\ref{thm:seam}). The schedule reuses the
calibration samples and preserves the inner objective. In the reported ternary DBF experiments, ICBQ has lower
PPL than Sequential CBQ on the seven main models; it also yields finite PPL for Qwen$3$-$8$B and
Llama-$3$-$8$B configurations with severe Sequential-CBQ degradation. The GPTQ results show that the schedule
can also be used with W3 and W2g128. These results suggest that schedule design is a useful additional
consideration for low-bit weight-only PTQ.

\paragraph{Limitations.}

Even with interleaving, the method remains block-local, so it does not fully capture the behavior of the truly
global end-to-end output error. Bridging this gap---both algorithmically and theoretically---is a key
direction for future work. In addition, our experiments focus on language models. Whether the same
schedule-level gains transfer to other families (e.g., vision models) remains open.

\paragraph{Broader impacts.} Potential positive impact: ICBQ lowers the memory and energy footprint of
deploying large language models at sub-two-bit precision, which enables wider access to LLM inference on
commodity hardware and reduces the carbon cost of inference-heavy workloads. Potential negative impact: a
cheaper deployment pipeline could lower the cost barrier for malicious or low-quality deployment of LLMs;
however, ICBQ is a pure quantization schedule that does not alter model behavior beyond reconstruction error,
so the risk profile matches that of existing PTQ methods such as GPTQ~\citep{frantar2022gptq} and
CBQ~\citep{ding2025cbq}. We view the net impact as positive within the standard PTQ risk envelope.

\bibliographystyle{plainnat}
\bibliography{ref}

\newpage
\appendix
\section*{Appendix Contents}
\begin{itemize}[leftmargin=1.5em,itemsep=0.15em,topsep=0.2em]
  \item Appendix~\ref{app:pseudo}: algorithm and schedule details.
  \item Appendix~\ref{app:assumptions}: formal assumptions, theorem statements, proofs, and related theory diagnostics.
  \item Appendix~\ref{app:exp-details}: experimental and implementation details.
  \item Appendix~\ref{app:zeroshot}: supplementary empirical results.
\end{itemize}

\section*{Notation}
\addcontentsline{toc}{section}{Notation}
\smallskip
\noindent\begin{tabular}{@{}>{\raggedright\arraybackslash}p{0.23\linewidth}p{0.72\linewidth}@{}}
  \toprule
  Symbol & Meaning \\
  \midrule
  $L$ & Number of transformer blocks, indexed by $b\in\{0,\dots,L{-}1\}$. \\
  $K$ & Chunk size used by ICBQ. \\
  $S(L,K)$ & Number of chunk seams, $S(L,K)=\lceil L/K\rceil-1$. \\
  $cK{-}1$ & Pair index of the $c$-th seam, corresponding to boundary pair $(cK{-}1,cK)$. \\
  $f_b^\star$ & The fp16 teacher block at depth $b$. \\
  $f_b^{\theta_b}$ & The current student block at depth $b$ with parameter state $\theta_b$. \\
  $Q$ & Inner one-block weight-only quantizer, instantiated as DBF or GPTQ. \\
  $X_b^\star$ & Teacher activation at depth $b$ on the calibration sample under discussion. \\
  $\hat X_b$ & Student activation at depth $b$. \\
  $e_b$ & Depth-$b$ activation mismatch, $e_b=\hat X_b-X_b^\star$. \\
  $m_i$ & Midpoint residual of pair $(i,i{+}1)$, $m_i=\|\hat X_{i+1}-X_{i+1}^\star\|_F$. \\
  $u_i, v_i$ & Post-pair midpoint residuals for Sequential CBQ and ICBQ, respectively. \\
  $\mathcal{T}_b$ & Finite set of teacher and student activations at depth $b$ visited by either schedule. \\
  $\Theta_b$ & Finite set of parameter states assumed by block $b$ during either schedule. \\
  $\rho$ & Uniform teacher-block Lipschitz constant on the visited activation sets. \\
  $\gamma$ & Per-sample midpoint contraction factor certified for each actual CBQ call. \\
  $\epsilon_{\mathrm{sub}}$ & Common residual envelope covering local representation error and CBQ fixed-point error. \\
  $\tau$ & Shorthand $\tau=\gamma\rho$; the proof assumes $\tau<1$. \\
  $R,\widetilde R$ & Recurrence forcing constants, $R=(1+\gamma)\epsilon_{\mathrm{sub}}$ and $\widetilde R=(1+\gamma+\gamma^2)\epsilon_{\mathrm{sub}}$. \\
  $s(i)$ & Number of seam indices at most $i$. \\
  $\mathcal{U}_i,\mathcal{V}_i$ & Deterministic upper bounds on sequential and interleaved post-pair midpoint residuals. \\
  $\mathcal{B}_L^{\mathrm{seq}},\mathcal{B}_L^{\mathrm{int}}$ & Sequential and interleaved depth-$L$ error upper bounds. \\
  $C(L,K,\rho,\gamma)$ & Additive residual constant in the formal ICBQ depth-$L$ bound. \\
  \bottomrule
\end{tabular}
\smallskip

\section{Algorithm and Schedule Details}\label{app:pseudo}

We record the full ICBQ schedule as two algorithms. The outer driver (Algorithm~\ref{alg:icbq-driver})
implements the progressive weight-only quantization loop: for each block it optionally runs prefit, applies
the inner quantizer $Q$, advances the activation streams, and at every chunk boundary calls the inner
refinement followed by an end-of-chunk re-roll. The inner refinement (Algorithm~\ref{alg:icbq}) implements the
interleaved CBQ call: given the window range $[w_{\min},w_{\max})$ produced by the outer driver it starts with
an in-pass re-roll through blocks $0,\dots,w_{\min}{-}1$ and then sweeps through the valid pair indices,
running at each pair a rollback-safeguarded Adam update on the plain-MSE two-block CBQ loss
$\mathcal{L}_{\mathrm{cbq}}$ followed by a one-layer advance of the student and teacher streams. Here
$\mathcal{L}_{\mathrm{pre}}$ denotes the one-block DBF prefit loss, and $\mathcal{L}_{\mathrm{cbq}}$ denotes
the plain-MSE two-block reconstruction loss used in our CBQ-style adaptation.

\begin{algorithm}[H]
  \caption{Interleaved Cross-Block Quantization (outer progressive driver).}
  \label{alg:icbq-driver}
  \begin{algorithmic}[1]
    \Require calibration batches $\mathcal{X}$, teacher blocks
    $\{f_{b}^{\star}\}_{b=0}^{L-1}$, inner quantizer $Q$, chunk size
    $K$, prefit steps $T_{\mathrm{pre}}$, inner-quantizer
    budget $T_Q$.
    \State \textbf{Save} stored inputs $\{X_{0}^{(s)}\}_{s}$.
    \State Initialize $\hat{X}_{0}\leftarrow X_{0}^{\star}\leftarrow\{X_{0}^{(s)}\}_{s}$, $b_{\mathrm{cs}}\leftarrow 0$.
    \For{$b=0,\dots,L-1$}
      \If{$Q=\mathrm{DBF}$}
        \State \textbf{Prefit}: run $T_{\mathrm{pre}}$ AdamW steps on $\mathcal{L}_{\mathrm{pre}}(\theta_{b})$ in float32.
      \EndIf
      \State \textbf{Quantize}: run the inner quantizer $Q$ with budget $T_Q$ on block $b$.
      \State Advance: $\hat{X}_{b+1}\leftarrow f_{b}^{Q}(\hat{X}_{b})$; $X_{b+1}^{\star}\leftarrow f_{b}^{\star}(X_{b}^{\star})$.
      \If{$b+1-b_{\mathrm{cs}}\ge K$ \textbf{or} $b=L-1$}
        \State $w_{\min}\leftarrow \max(0,b_{\mathrm{cs}}-1), ~~ w_{\max}\leftarrow \min(b+1,L-1)$
        \State Call Algorithm~\ref{alg:icbq} with $[w_{\min},w_{\max})$.
        \State \textbf{Re-roll}: recompute $\hat{X}_{0\to b+1}$ and $X_{0\to b+1}^{\star}$ from stored inputs.
        \State $b_{\mathrm{cs}}\leftarrow b+1$.
      \EndIf
    \EndFor
    \State \Return quantized model $\{f_{b}^{Q}\}_{b}$ with refined parameters.
  \end{algorithmic}
\end{algorithm}

\vspace{0.6em}

\begin{algorithm}[H]
  \caption{Interleaved CBQ refinement on a window range
    $[w_{\min},w_{\max})$.}
  \label{alg:icbq}
  \begin{algorithmic}[1]
    \Require Student state
    $\{f_{b}^{\theta_{b}}\}_{b\le w_{\max}}$, teacher state
    $\{f_{b}^{\star}\}_{b\le w_{\max}}$, stored calibration inputs
    $\{X_{0}^{(s)}\}_{s}$.
    \State Re-roll: advance student and teacher copies in lockstep through
    blocks $0,\dots,w_{\min}{-}1$ from the stored inputs.
    \For{$i=w_{\min},\dots,w_{\max}-1$}
      \State Compute teacher targets $\{Y_{i+2}^{(s),\star}\}_{s}$ by
      passing the current teacher state through $f_{i}^{\star},f_{i+1}^{\star}$.
      \State Minimize the two-block CBQ loss $\mathcal{L}_{\mathrm{cbq}}$ on
      $(\theta_{i},\theta_{i+1})$ with Adam; rollback on non-improvement.
      \State Advance student and teacher by one layer
      (through $f_{i}^{\theta_{i}}$ and $f_{i}^{\star}$).
    \EndFor
  \end{algorithmic}
\end{algorithm}

\section{Formal Theory and Proofs}
\label{app:assumptions}
\label{app:proof}

This appendix collects the three structural hypotheses (Assumptions~\ref{asm:lip}, \ref{asm:cbq},
and~\ref{asm:reroll}), two auxiliary remarks, and the formal theorem and corollary statements used by the
detailed proof below. These formal statements correspond to the informal seam-amortized contraction theorem
and amortization-gain corollary in the main paper.

\subsection{Assumptions}
\label{app:asm-formal}
\paragraph{Preamble.} Fix a calibration set $\mathcal{X} = \{x^{(n)}\}_{n=1}^{N}$ and a chunk size $K$. Let
$\mathcal{S}_{\mathrm{seq}}$ denote the fixed two-block Sequential~CBQ adaptation defined in
Section~\ref{sec:prelim}, and let $\mathcal{S}_{\mathrm{int}}$ denote ICBQ with this chunk size; both
schedules are applied to the same calibration set, use the same fp16 teacher blocks
$\{f_{b}^{\star}\}_{b=0}^{L-1}$, and use the same inner quantizer $Q$ (including the same DBF prefit objective
$\mathcal{L}_{\mathrm{pre}}$ when $Q = \mathrm{DBF}$) and the same rollback-safeguarded Adam budget for every
CBQ subproblem. The constants $\rho,\gamma,\epsilon_{\mathrm{sub}}$ below are certified for the schedules at
this fixed $K$; the theorem therefore applies pointwise to any $K$ for which these certificates hold. A claim
uniform over multiple chunk sizes requires the same constants to hold uniformly over those choices of $K$.
Both schedules execute deterministic, finitely many forward passes and parameter updates; hence, for each
depth $b$, the set of depth-$b$ student activations $\hat{X}_{b}^{(n)}$ arising anywhere in either schedule is
finite. Denote
\begin{equation}
  \mathcal{T}_{b} \coloneqq \{X_{b}^{\star,(n)}:n\in[N]\} \cup\{\hat{X}_{b}^{(n)}\text{ arising in }\mathcal{S}_{\mathrm{seq}}\} \cup\{\hat{X}_{b}^{(n)}\text{ arising in }\mathcal{S}_{\mathrm{int}}\} \subset \mathbb{R}^{T\times d_{\mathrm{hidden}}}.
\end{equation}

\begin{assumption}[Lipschitz teacher blocks]
  \label{asm:lip}
  There exists a constant $\rho > 0$, independent of the depth $b$ and of the schedule, such that
  \begin{equation}
    \|f_{b}^{\star}(x)-f_{b}^{\star}(y)\|_{F} \le \rho\,\|x-y\|_{F} ~~ \text{for all }x,y \in \mathcal{T}_{b},\ b \in \{0,\dots,L{-}1\}.
    \label{eq:asm-lip}
  \end{equation}
  The non-expansive regime $\rho \le 1$ captures empirical residual-stream behavior of trained transformer
  blocks; our proofs only require $\tau \coloneqq \gamma\rho < 1$, so locally expansive blocks ($\rho > 1$)
  are admissible if $\gamma$ is small (Appendix~\ref{app:stability}).
\end{assumption}

For each depth $i$, let $\Theta_{i}$ denote the finite set of parameter states assumed by block $i$ at any
intermediate state of either schedule (initial fp16 teacher parameters, post-prefit, post-inner-quantizer,
post-CBQ-update, and any other pre-pass state that appears as an input to a CBQ subproblem).

\begin{assumption}[Contractive two-block CBQ subproblem]
  \label{asm:cbq}
  There exist constants $\gamma \in (0,1)$ and $\epsilon_{\mathrm{sub}} \ge 0$, independent of the depth and
  of the schedule, such that:

  \noindent\emph{(Local representation envelope.)} For every $i$, every $\theta_{i} \in \Theta_{i}$, and every
  $z \in \mathcal{T}_{i}$,
  \begin{equation}
    \bigl\|f_{i}^{\theta_i}(z)-f_{i}^{\star}(z)\bigr\|_{F} \le \epsilon_{\mathrm{sub}}.
    \label{eq:local-envelope}
  \end{equation}

  \noindent\emph{(Midpoint contraction.)} For every actual CBQ call made by either schedule on a pair
  $(i,i{+}1)$, let $(\theta_{i}^{\mathrm{pre}},\theta_{i+1}^{\mathrm{pre}})$ be the pre-pass block states and
  $\{\hat{X}_{i}^{(n)}\}_{n=1}^{N}$ the student inputs; let
  $(\theta_{i}^{\mathrm{post}},\theta_{i+1}^{\mathrm{post}})$ be the shared post-pass pair returned by the
  rollback-safeguarded Adam optimizer applied to the plain-MSE two-block CBQ loss
  $\mathcal{L}_{\mathrm{cbq}}$. Then, for every sample $n$ and every $\hat{X}_{i} = \hat{X}_{i}^{(n)}$,
  writing
  \begin{equation}
    r^{\mathrm{in}} \coloneqq \bigl\|f_{i}^{\theta_{i}^{\mathrm{pre}}}(\hat{X}_{i}) - X_{i+1}^{\star}\bigr\|_{F}, ~~ r^{\mathrm{out}} \coloneqq \bigl\|f_{i}^{\theta_{i}^{\mathrm{post}}}(\hat{X}_{i}) - X_{i+1}^{\star}\bigr\|_{F},
  \end{equation}
  we have
  \begin{equation}
    r^{\mathrm{out}} \le \gamma\, r^{\mathrm{in}}+\epsilon_{\mathrm{sub}}.
    \label{eq:asm-cbq}
  \end{equation}
\end{assumption}

Eq.~\eqref{eq:asm-cbq} is the linear-convergence form of a first-order method with a nonzero fixed-point
residual; Eq.~\eqref{eq:local-envelope} is the local representation envelope needed to transition from a
depth-$i$ input error to a depth-$(i{+}1)$ midpoint error. Three points deserve emphasis. First,
Eq.~\eqref{eq:asm-cbq} is \emph{per-sample}: the optimizer returns a single parameter pair, but the assumption
requires the induced midpoint residual to contract uniformly over every presented sample, which is strictly
stronger than contraction on the sample-averaged residual. Second, the rollback safeguard only certifies the
non-increase of the monitored two-block loss; the per-sample midpoint contraction in Eq.~\eqref{eq:asm-cbq} is
an additional empirical certificate (Remark~\ref{rem:rollback-midpoint-certificate} below). Third, if separate
constants are preferred, $\epsilon_{\mathrm{sub}}$ can be split into a local representation radius
$\epsilon_{\mathrm{rep}}$ (in Eq.~\eqref{eq:local-envelope}) and a CBQ fixed-point radius
$\epsilon_{\mathrm{opt}}$ (in Eq.~\eqref{eq:asm-cbq}), and every subsequent inequality holds with
$\epsilon_{\mathrm{sub}}$ replaced by $\max(\epsilon_{\mathrm{rep}},\epsilon_{\mathrm{opt}})$. The aggregate
constant $\gamma$ is estimated empirically in Appendix~\ref{app:gamma-measure}.

\begin{remark}[Midpoint contraction is an additional certificate, not a consequence of rollback]
  \label{rem:rollback-midpoint-certificate}
  The rollback safeguard is applied to the two-block output reconstruction loss $\mathcal{L}_{\mathrm{cbq}}$,
  which couples blocks $i$ and $i{+}1$ jointly through $f_{i+1}^{\theta_{i+1}}\circ f_{i}^{\theta_{i}}$. The
  midpoint residual depends only on block $i$, so rollback does \emph{not} by itself imply the per-sample
  midpoint contraction in Eq.~\eqref{eq:asm-cbq}. Assumption~\ref{asm:cbq} is therefore an additional local
  certificate of the implemented rollback-safeguarded Adam on the finite collection of CBQ calls actually made
  by the schedules; the proof of Theorem~\ref{thm:seam-formal} uses only this midpoint certificate together
  with the local representation envelope, not the fact that the two-block loss decreases. The empirical study
  in Appendix~\ref{app:gamma-measure} provides an aggregate diagnostic of this behavior, rather than a uniform
  certificate.
\end{remark}

\begin{assumption}[Exact re-roll]
  \label{asm:reroll}
  Every re-roll performed by either schedule has an associated prefix length $d$: for the
  \emph{in-pass re-roll} at the start of every interleaved refinement call (line~1 of
  Algorithm~\ref{alg:icbq}), $d = w_{\min}$; for the \emph{end-of-chunk re-roll} in the outer driver,
  performed after block $b$ is processed, $d = b+1$. The re-roll recomputes, for each calibration sample $n$
  and every target depth $b' \in \{0,\dots,d\}$,
  \begin{equation}
    \hat{X}_{b'}^{(n)} = \bigl(f_{b'-1}^{\theta_{b'-1}} \circ \cdots \circ f_{0}^{\theta_{0}}\bigr) \bigl(X_{0}^{(n)}\bigr), ~~ X_{b'}^{\star,(n)} = \bigl(f_{b'-1}^{\star} \circ \cdots \circ f_{0}^{\star}\bigr) \bigl(X_{0}^{(n)}\bigr),
  \end{equation}
  where $\theta_{0},\dots,\theta_{b'-1}$ are the \emph{current} parameter states of blocks $0,\dots,b'{-}1$ at
  the moment of the re-roll (with $b'=0$ denoting the empty composition). In the same symbolic sense, every
  single-block forward pass executed elsewhere in either schedule is treated as an exact function application
  of its block map.
\end{assumption}

In our implementation, both re-rolls and every in-loop advance step are deterministic forward passes through
their current prefixes using the stored layer-$0$ inputs. Assumption~\ref{asm:reroll} abstracts away
floating-point non-associativity in the symbolic identities used by the proof. The in-pass re-roll identifies
the repeated seam input, while the end-of-chunk re-roll keeps the states and activations included in
$\Theta_{i}$ and $\mathcal{T}_{i}$ consistent with the current prefix.

\subsection{Formal theorem, corollary, and remarks}
\label{app:thm-formal}

For a schedule $\mathcal{S}\in\{\mathcal{S}_{\mathrm{seq}},\mathcal{S}_{\mathrm{int}}\}$, write
$e_{b}^{\mathcal{S}}\coloneqq \hat{X}_{b}^{\mathcal{S}}-X_{b}^{\star}$ for the depth-$b$ activation mismatch
on a calibration sample. The expectation $\mathbb{E}$ in the theorem statements is the empirical average over
the fixed calibration set $\mathcal{X}$.

Under Assumptions~\ref{asm:lip}, \ref{asm:cbq}, and~\ref{asm:reroll} and $\tau<1$, define the
\emph{sequential depth-$L$ bound}
\begin{equation}
  \mathcal{B}_{L}^{\mathrm{seq}} \coloneqq \rho\,(1+\gamma)\,\frac{1-\tau^{L-1}}{1-\tau}\,\epsilon_{\mathrm{sub}} + \epsilon_{\mathrm{sub}},
  \label{eq:B-seq-def}
\end{equation}
which is the sequential upper bound used for the comparison (derivation in Appendix~\ref{app:finale}).

\begin{theorem}[Seam-amortized error contraction, formal]
  \label{thm:seam-formal}
  Under Assumptions~\ref{asm:lip}, \ref{asm:cbq}, and~\ref{asm:reroll} and $\tau=\gamma\rho<1$, for every
  $L \ge 2$ and every $1 \le K \le L$,
  \begin{equation}
    \mathbb{E}\bigl\|e_{L}^{\mathrm{int}}\bigr\|_{F} \le \gamma^{S(L,K)}\,\mathcal{B}_{L}^{\mathrm{seq}} + C(L,K,\rho,\gamma)\,\epsilon_{\mathrm{sub}},
    \label{eq:thm-seam-formal}
  \end{equation}
  where $S(L,K) = \lceil L/K\rceil-1$, and one admissible (generally non-tight) choice of the residual
  constant is
  \begin{equation}
    \begin{aligned}
      C(L,K,\rho,\gamma) &= \bigl(1-\gamma^{S(L,K)}\bigr)\\
      &~ + \Bigl[(1+\gamma+\gamma^{2}) - \gamma^{S(L,K)}(1+\gamma)\Bigr]\, \rho\,\frac{1-\tau^{L-1}}{1-\tau}.
    \end{aligned}
    \label{eq:thm-seam-C-admissible}
  \end{equation}
  In particular,
  \begin{equation}
    C(L,K,\rho,\gamma) \le 1+\frac{(1+\gamma+\gamma^{2})\,\rho}{1-\tau} ~~ \text{uniformly in }L,K,
    \label{eq:thm-seam-C}
  \end{equation}
  and for the standard stability regime $\rho \le 1$ this simplifies to
  $C \le 1+(1+\gamma+\gamma^{2})/(1-\tau)$; for $K = L$ the leading factor is $\gamma^{0} = 1$ (one chunk, no
  seams), and for $K = 2$ it contracts by $\gamma^{\lceil L/2\rceil-1}$. The per-seam factor $\gamma$ is tight
  at the recurrence level (Proposition~\ref{prop:tightness}).
\end{theorem}

\begin{remark}[Non-tightness of the displayed residual constant]
  \label{rem:C-nontight}
  The constant in Eq.~\eqref{eq:thm-seam-C-admissible} is chosen for a simple closed-form comparison with the
  sequential bound $\mathcal{B}_{L}^{\mathrm{seq}}$. It is one admissible upper-bound constant produced by the
  bookkeeping in Proposition~\ref{prop:chain}---specifically by the pointwise estimate
  $\beta_{k}\le\widetilde{R}$ used in Eq.~\eqref{eq:v-first-bound}---and is generally not the smallest
  constant one could write down. In particular, when $K = L$ there is a single chunk and no seam ($S = 0$), so
  ICBQ and sequential CBQ coincide as schedules at the recurrence level; the displayed residual constant
  nevertheless reduces only to $C(L,L,\rho,\gamma) = \gamma^{2}\rho(1-\tau^{L-1})/(1-\tau)$ rather than to
  zero, because the same $\beta_{k}\le\widetilde{R}$ estimate is applied to non-seam pairs in
  Eq.~\eqref{eq:v-first-bound}. A sharper seam-aware comparison that vanishes in the no-seam case is recorded
  in Proposition~\ref{prop:chain-sharp} below.
\end{remark}

\begin{remark}[Upper-bound semantics]
  \label{rem:ub-semantics}
  The formal ICBQ depth-$L$ bound in Theorem~\ref{thm:seam-formal} compares ICBQ's expected depth-$L$ error
  against the explicit sequential upper bound $\mathcal{B}_{L}^{\mathrm{seq}}$ derived from the same
  recurrences, rather than against the realized sequential error. The comparison is therefore a bound-level
  statement: the leading propagated term carries the factor $\gamma^{S}$, while all additive terms that do not
  carry this factor are collected in $C\,\epsilon_{\mathrm{sub}}$.
\end{remark}

\begin{corollary}[Geometric decay in the number of seams, formal]
  \label{cor:geom-formal}
  When $\mathcal{B}_{L}^{\mathrm{seq}} > 0$, define the \emph{amortization gain} of ICBQ over sequential CBQ
  as $G(L,K) \coloneqq \mathcal{B}_{L}^{\mathrm{seq}}/\mathcal{B}_{L}^{\mathrm{int}}$, where
  $\mathcal{B}_{L}^{\mathrm{int}} = \gamma^{S(L,K)}\mathcal{B}_{L}^{\mathrm{seq}} + C(L,K,\rho,\gamma)\,\epsilon_{\mathrm{sub}}$
  is the explicit right-hand side of Eq.~\eqref{eq:thm-seam-formal}. Under the hypotheses of
  Theorem~\ref{thm:seam-formal}, if for some $\delta \ge 0$,
  \begin{equation}
    C(L,K,\rho,\gamma)\,\epsilon_{\mathrm{sub}} \le \delta\,\gamma^{S(L,K)}\,\mathcal{B}_{L}^{\mathrm{seq}},
  \end{equation}
  then
  \begin{equation}
    G(L,K) \ge \frac{\gamma^{-S(L,K)}}{1+\delta}.
    \label{eq:amort-finite}
  \end{equation}
  Consequently, for any family of instances in which
  $C(L,K,\rho,\gamma)\,\epsilon_{\mathrm{sub}} = o\bigl(\gamma^{S(L,K)}\mathcal{B}_{L}^{\mathrm{seq}}\bigr)$,
  \begin{equation}
    G(L,K) \ge \gamma^{-S(L,K)}\bigl(1-o(1)\bigr) = \gamma^{-(\lceil L/K\rceil-1)}\bigl(1-o(1)\bigr).
    \label{eq:amort}
  \end{equation}
\end{corollary}

\begin{proof}
  By Theorem~\ref{thm:seam-formal}, the interleaved upper bound is
  $\mathcal{B}_{L}^{\mathrm{int}} = \gamma^{S(L,K)}\mathcal{B}_{L}^{\mathrm{seq}} + C(L,K,\rho,\gamma)\,\epsilon_{\mathrm{sub}}$
  with $C\ge 0$, so
  $\mathcal{B}_{L}^{\mathrm{int}} \le (1+\delta)\gamma^{S(L,K)}\mathcal{B}_{L}^{\mathrm{seq}}$ under the
  hypothesis, and $\mathcal{B}_{L}^{\mathrm{int}} \ge \gamma^{S(L,K)}\mathcal{B}_{L}^{\mathrm{seq}} > 0$.
  Dividing yields $G(L,K) \ge 1/\bigl((1+\delta)\gamma^{S(L,K)}\bigr)$, which is Eq.~\eqref{eq:amort-finite}.
  The $o(1)$ version follows by setting
  $\delta_{L,K} \coloneqq C\epsilon_{\mathrm{sub}}/(\gamma^{S(L,K)}\mathcal{B}_{L}^{\mathrm{seq}})$.
\end{proof}

\subsection{Proof Overview}
\label{app:proof-overview}

We use the notation of the seam-amortized contraction statement. For a fixed calibration sample
$x \in \mathcal{X}$, write $e_{b} \coloneqq \hat{X}_{b}-X_{b}^{\star}$ for the activation mismatch at depth
$b$, and
\begin{equation}
  m_{i} \coloneqq \|\hat{X}_{i+1}-X_{i+1}^{\star}\|_{F} = \|e_{i+1}\|_{F}, ~~ i \in \{-1,0,\dots,L{-}2\},
  \label{eq:midpoint-def}
\end{equation}
for the midpoint residual of pair $(i,i{+}1)$, which equals the activation mismatch at depth $i{+}1$. We adopt
$i = -1$ as a formal boundary index corresponding to the common calibration input, so that
$m_{-1} = \|\hat{X}_{0}-X_{0}^{\star}\|_{F} = 0$ since both schedules start from
$\hat{X}_{0} = X_{0}^{\star}$; this zero initial condition is used as the base case in
Lemma~\ref{lem:per-pair} below. We use the shorthand $\tau\coloneqq\gamma\rho$ throughout and assume, as in
Theorem~\ref{thm:seam-formal}, that $\tau<1$. Whenever the same symbol could arise at several moments of the
algorithm, it is tied to the specific schedule and time stated in the surrounding sentence; superscripts such
as ``seq'', ``int'', ``pre'', ``post'', and ``out'' are used only where the distinction is needed for an
inequality. All statements below are first pathwise for this fixed calibration sample, and are averaged only
in Section~\ref{app:finale}.

The structure of the proof is the following.
\begin{enumerate}[topsep=2pt,itemsep=0pt]
  \item Section~\ref{app:per-pair} (Lemma~\ref{lem:per-pair}): a uniform one-step
  recursion on $m_{i}$ that holds under both schedules for any pair receiving a single CBQ pass.
  \item Section~\ref{app:seam-structural}
  (Lemma~\ref{lem:seam-structural}): a formal record of the structural fact that every chunk seam is visited by the inner refinement of two consecutive chunks.
  \item Section~\ref{app:single-seam} (Lemma~\ref{lem:seam}): isolates the contribution
  of a single chunk seam, where ICBQ's two-pass certificate yields one extra factor $\gamma$ relative to a single pass.
  \item Section~\ref{app:chain} (Lemma~\ref{lem:recur} and
  Propositions~\ref{prop:chain}--\ref{prop:chain-sharp}): chains the per-pair and seam recurrences along the depth axis, producing closed-form upper bounds $\mathcal{U}_{i}, \mathcal{V}_{i}$ on the midpoint residual at each pair.
  \item Section~\ref{app:finale}: extracts the final-prefix identity for
  $\hat{X}_{L-1}$ via Lemma~\ref{lem:T-monotone} (final-visit-time monotonicity) and Lemma~\ref{lem:final-prefix} (final-prefix invariance), and then assembles Theorem~\ref{thm:seam-formal}.
  \item Section~\ref{app:tightness} (Proposition~\ref{prop:tightness}): the per-seam
  factor $\gamma^{S}$ is tight at the recurrence level.
  \item Section~\ref{app:stability} (Proposition~\ref{prop:stability}): extension to the
  locally expansive regime $\rho > 1$ under $\tau < 1$.
\end{enumerate}

\subsection{Per-pair recursion}
\label{app:per-pair}

\begin{lemma}[Per-pair recursion]
  \label{lem:per-pair}
  Fix a schedule $\mathcal{S} \in \{\mathcal{S}_{\mathrm{seq}},\mathcal{S}_{\mathrm{int}}\}$. For every pair
  $(i,i{+}1)$ that receives a single CBQ pass under $\mathcal{S}$, the midpoint residual obeys
  \begin{equation}
    m_{i}^{\mathrm{out}} \le \gamma\rho\,m_{i-1}^{\mathrm{out}} + (1+\gamma)\,\epsilon_{\mathrm{sub}}, ~~ i\ge 0,
    \label{eq:per-pair-lemma}
  \end{equation}
  with $m_{-1}^{\mathrm{out}} \coloneqq 0$ (both schedules start from the common calibration input
  $\hat{X}_{0} = X_{0}^{\star}$). The inequality holds pathwise for every calibration sample.
\end{lemma}

\begin{proof}
  Fix an arbitrary calibration sample $n \in \{1,\dots,N\}$; we suppress the superscript $^{(n)}$ in what
  follows since every quantity below is evaluated on this single sample. Let $\hat{X}_{i}$ and $X_{i}^{\star}$
  denote the student and teacher activations at depth $i$ in $\mathcal{S}$ immediately before pair $(i,i{+}1)$
  is processed. By Assumption~\ref{asm:lip}, every student activation arising in either schedule at depth $i$
  lies in $\mathcal{T}_{i}$, as does $X_{i}^{\star}$; hence $\hat{X}_{i},X_{i}^{\star} \in \mathcal{T}_{i}$.
  By Eq.~\eqref{eq:midpoint-def}, $m_{i-1}^{\mathrm{out}} = \|\hat{X}_{i}-X_{i}^{\star}\|_{F}$. Let
  $\theta_{i}^{\mathrm{pre}}$ and $\theta_{i+1}^{\mathrm{pre}}$ denote the parameters of blocks $i$ and
  $i{+}1$ at that moment (obtained from the prefit step when present, the inner quantizer, and any earlier CBQ
  updates that $\mathcal{S}$ has performed on those blocks). Both states arise at an intermediate step of
  $\mathcal{S}$, so $\theta_{i}^{\mathrm{pre}} \in \Theta_{i}$ and
  $\theta_{i+1}^{\mathrm{pre}} \in \Theta_{i+1}$ by the construction of $\Theta_{i},\Theta_{i+1}$ in
  Assumption~\ref{asm:cbq}. Thus the CBQ call currently being processed is one of the actual calls covered by
  the midpoint contraction in Assumption~\ref{asm:cbq}: its pre-pass states lie in
  $\Theta_{i} \times \Theta_{i+1}$ and its sample input $\hat{X}_{i}$ lies in $\mathcal{T}_{i}$.

  \smallskip \noindent\emph{Bound on the pre-pass residual.} The pre-pass midpoint residual of pair
  $(i,i{+}1)$ is
  \begin{equation*}
    r^{\mathrm{in}} = \bigl\|f_{i}^{\theta_{i}^{\mathrm{pre}}}(\hat{X}_{i})-X_{i+1}^{\star}\bigr\|_{F} = \bigl\|f_{i}^{\theta_{i}^{\mathrm{pre}}}(\hat{X}_{i})-f_{i}^{\star}(X_{i}^{\star})\bigr\|_{F},
  \end{equation*}
  where we used $X_{i+1}^{\star} = f_{i}^{\star}(X_{i}^{\star})$. By the triangle inequality,
  \begin{equation}
    \begin{aligned}
      r^{\mathrm{in}}
      &\le \underbrace{\bigl\|f_{i}^{\theta_{i}^{\mathrm{pre}}}(\hat{X}_{i})-f_{i}^{\star}(\hat{X}_{i})\bigr\|_{F}}_{\le\,\epsilon_{\mathrm{sub}}\text{ by Eq.~\eqref{eq:local-envelope}}}
      + \underbrace{\bigl\|f_{i}^{\star}(\hat{X}_{i})-f_{i}^{\star}(X_{i}^{\star})\bigr\|_{F}}_{\le\,\rho\,\|\hat{X}_{i}-X_{i}^{\star}\|_{F}=\rho\,m_{i-1}^{\mathrm{out}}\text{ by Asm.~\ref{asm:lip}}} \\
      &\le \rho\,m_{i-1}^{\mathrm{out}}+\epsilon_{\mathrm{sub}}.
    \end{aligned}
    \label{eq:rin-bound}
  \end{equation}
  where the first bound applies the local envelope in Eq.~\eqref{eq:local-envelope} (valid because
  $\theta_{i}^{\mathrm{pre}} \in \Theta_{i}$ and $\hat{X}_{i} \in \mathcal{T}_{i}$) and the second uses
  Assumption~\ref{asm:lip} (valid because $\hat{X}_{i},X_{i}^{\star} \in \mathcal{T}_{i}$).

  \smallskip \noindent\emph{Applying the midpoint contraction.} Using Eq.~\eqref{eq:asm-cbq} on the single
  pass on pair $(i,i{+}1)$,
  \begin{equation}
    m_{i}^{\mathrm{out}} = r^{\mathrm{out}} \le \gamma\,r^{\mathrm{in}}+\epsilon_{\mathrm{sub}} \le \gamma\bigl(\rho\,m_{i-1}^{\mathrm{out}}+\epsilon_{\mathrm{sub}}\bigr)+\epsilon_{\mathrm{sub}} = \gamma\rho\,m_{i-1}^{\mathrm{out}}+(1+\gamma)\,\epsilon_{\mathrm{sub}},
    \label{eq:rout-bound}
  \end{equation}
  where the first equality uses
  $m_{i}^{\mathrm{out}} = \|\hat{X}_{i+1}-X_{i+1}^{\star}\|_{F} = \|f_{i}^{\theta_{i}^{\mathrm{post}}}(\hat{X}_{i})-X_{i+1}^{\star}\|_{F} = r^{\mathrm{out}}$
  (since the student advance from $\hat{X}_{i}$ to $\hat{X}_{i+1}$ uses the post-update block $i$). This is
  exactly Eq.~\eqref{eq:per-pair-lemma}. The base case $m_{-1}^{\mathrm{out}} = 0$ holds because both
  schedules share the common calibration input $\hat{X}_{0} = X_{0}^{\star}$; the convention
  $m_{-1}^{\mathrm{out}} = \|\hat{X}_{0}-X_{0}^{\star}\|_{F}$ follows from Eq.~\eqref{eq:midpoint-def} with
  $i = -1$.
\end{proof}

\subsection{Structural fact: every seam pair is visited twice}
\label{app:seam-structural}

Before isolating the seam contraction quantitatively, we first record, as a separate lemma, the structural
``visited twice'' property of the ICBQ schedule.

\begin{lemma}[Seam structural property]
  \label{lem:seam-structural}
  Let $L \ge 2$ and $1 \le K \le L$, with $S = S(L,K) = \lceil L/K\rceil-1$ the number of chunk seams. For
  every $c \in \{1,\dots,S\}$:
  \begin{enumerate}[label=\textup{(\alph*)},topsep=2pt,itemsep=0pt]
    \item the seam index $cK{-}1$ is a valid pair index, i.e.\ $cK{-}1 \le L{-}2$;
    \item the boundary pair $(cK{-}1,cK)$ is the \emph{last} pair processed in chunk
    $c$'s inner refinement, and the \emph{first} pair processed in chunk $c{+}1$'s inner refinement;
    \item under ICBQ, the boundary pair $(cK{-}1,cK)$ is therefore
    optimized exactly twice.
  \end{enumerate}
\end{lemma}

\begin{proof}
  \emph{Validity (a).} If $S = 0$ there is nothing to prove; assume henceforth $S \ge 1$. Write $L = qK+r$
  with $q = \lfloor L/K\rfloor$ and $0 \le r < K$. If $r = 0$ then $\lceil L/K\rceil = q$, $S = q{-}1 \ge 1$
  (so $q \ge 2$), and $SK = (q{-}1)K = L{-}K \le L{-}1$ since $K \ge 1$. If $r \ge 1$ then
  $\lceil L/K\rceil = q{+}1$, $S = q \ge 1$, and $SK = qK = L{-}r \le L{-}1$ since $r \ge 1$. In both cases
  $SK \le L{-}1$, hence for every $c \in \{1,\dots,S\}$ the seam index satisfies
  $cK{-}1 \le SK{-}1 \le L{-}2$, which proves (a).

  \emph{Pair-loop placement (b).} For chunk $c \in \{1,\dots,S\}$, the chunk-pointer satisfies
  $b_{\mathrm{cs}} = (c{-}1)K$ at the start of chunk $c$, and the chunk closes when the outer driver reaches
  $b = cK{-}1$ (i.e.\ $b{+}1-b_{\mathrm{cs}} = K$). By the chunk-closure window rule, the window range at this
  closure is $[w_{\min},w_{\max}) = [\,\max(0,(c{-}1)K{-}1),\,\min(cK,L{-}1)\,)$. Since $cK \le L{-}1$ by (a),
  $w_{\max} = cK$, hence the largest valid pair index in chunk $c$'s pair loop is $w_{\max}-1 = cK{-}1$, so
  the seam pair $(cK{-}1,cK)$ is the \emph{last} pair processed in chunk $c$'s inner refinement. Conversely,
  chunk $c{+}1$ starts with $b_{\mathrm{cs}} = cK$, so its window range begins at
  $w_{\min} = \max(0,cK{-}1) = cK{-}1$ (since $cK \ge K \ge 1$), making $(cK{-}1,cK)$ the \emph{first} pair
  processed in chunk $c{+}1$'s inner refinement.

  \emph{Multiplicity (c).} By (b), the seam pair $(cK{-}1,cK)$ is processed once at the end of chunk $c$'s
  pair loop and once at the start of chunk $c{+}1$'s pair loop, hence optimized exactly twice in total.
\end{proof}

\subsection{Single-seam lemma}
\label{app:single-seam}

The following lemma isolates the contribution of a single chunk seam where ICBQ and sequential CBQ differ.

\begin{lemma}[Single-seam contraction]
  \label{lem:seam}
  Fix $c \in \{1,\dots,S\}$ and consider the seam pair $(cK{-}1,cK)$. Let
  $\hat{X}_{cK-1} \in \mathcal{T}_{cK-1}$ be the student activation at depth $cK{-}1$ immediately before the
  first CBQ pass on this pair, and let
  $(\theta_{cK-1}^{\mathrm{pre}},\theta_{cK}^{\mathrm{pre}}) \in \Theta_{cK-1} \times \Theta_{cK}$ be the
  pre-first-pass parameters of blocks $cK{-}1$ and $cK$. For the actual ICBQ seam state considered in
  part~\textup{(ii)}, the second component is the fp16 teacher copy,
  $\theta_{cK}^{\mathrm{pre}} = \theta_{cK}^{\star}$; part~\textup{(i)} records the contraction certified for
  this first pass. Let
  \begin{equation}
    \widetilde{m} \coloneqq \bigl\|f_{cK-1}^{\theta_{cK-1}^{\mathrm{pre}}}(\hat{X}_{cK-1})-X_{cK}^{\star}\bigr\|_{F}.
  \end{equation}
  Then:
  \begin{enumerate}[topsep=2pt,itemsep=2pt]
    \item[\textup{(i)}] \emph{(One certified pass.)}
    The first CBQ pass on $(cK{-}1,cK)$ from this actual seam state leaves a post-pair midpoint residual
    \begin{equation}
      m_{cK-1}^{(1),\mathrm{out}} \le \gamma\,\widetilde{m}+\epsilon_{\mathrm{sub}}.
      \label{eq:seam-seq}
    \end{equation}
    \item[\textup{(ii)}] \emph{(Interleaved.)}
    Under ICBQ, the seam pair $(cK{-}1,cK)$ is visited \emph{twice} (Lemma~\ref{lem:seam-structural}): pass~$1$ is the last pair in chunk~$c$'s inner refinement, and pass~$2$ is the first pair in chunk~$c{+}1$'s inner refinement. Between the two passes the outer driver executes the end-of-chunk re-roll at the close of chunk~$c$, the prefit/inner-quantizer/one-layer-advance sequence on blocks $cK,\dots,\min\{(c{+}1)K{-}1,L{-}1\}$, and the in-pass re-roll opening chunk $c{+}1$'s refinement; none of these operations modifies the parameters of blocks $0,\dots,cK{-}1$. In this setting the post-seam midpoint residual satisfies
    \begin{equation}
      m_{cK-1}^{\mathrm{int,out}} \le \gamma^{2}\,\widetilde{m} + (1+\gamma)\,\epsilon_{\mathrm{sub}}.
      \label{eq:seam-int}
    \end{equation}
  \end{enumerate}
\end{lemma}

\begin{proof}
  \emph{Claim~(i).} The pre-pass state $(\theta_{cK-1}^{\mathrm{pre}},\theta_{cK}^{\mathrm{pre}})$ of the
  first seam pass is admissible by hypothesis, i.e.\ it lies in $\Theta_{cK-1} \times \Theta_{cK}$; the
  student activation $\hat{X}_{cK-1}$ presented to the pair lies in $\mathcal{T}_{cK-1}$. Since this is an
  actual CBQ call of the interleaved schedule, the midpoint contraction in Eq.~\eqref{eq:asm-cbq} applies
  directly and yields Eq.~\eqref{eq:seam-seq}.

  \smallskip \noindent\emph{Claim~(ii).} Write $\theta_{cK-1}^{(1),\mathrm{post}}$ and
  $\theta_{cK}^{(1),\mathrm{post}}$ for the states of blocks $cK{-}1$ and $cK$ after pass~$1$; Adam updates
  both block parameters jointly during the look-ahead pass, despite the pre-pass state of block $cK$ being the
  fp16 teacher copy $\theta_{cK}^{\star}$, which lies in $\Theta_{cK}$ by the definition preceding
  Assumption~\ref{asm:cbq}. Since the pre-pass pair
  $(\theta_{cK-1}^{\mathrm{pre}},\theta_{cK}^{\mathrm{pre}}) = (\theta_{cK-1}^{\mathrm{pre}},\theta_{cK}^{\star})$
  is admissible and $\hat{X}_{cK-1} \in \mathcal{T}_{cK-1}$, Assumption~\ref{asm:cbq} applied to the actual
  pass~$1$ provides
  $(\theta_{cK-1}^{(1),\mathrm{post}},\theta_{cK}^{(1),\mathrm{post}}) \in \Theta_{cK-1} \times \Theta_{cK}$
  together with the pathwise midpoint contraction
  \begin{equation}
    r^{(1)} \coloneqq \bigl\|f_{cK-1}^{\theta_{cK-1}^{(1),\mathrm{post}}}(\hat{X}_{cK-1})-X_{cK}^{\star}\bigr\|_{F} \le \gamma\,\widetilde{m}+\epsilon_{\mathrm{sub}}.
    \label{eq:seam-pass1}
  \end{equation}
  Observe that $r^{(1)}$ depends on block $cK$ only through the teacher activation
  $X_{cK}^{\star} = f_{cK-1}^{\star}(X_{cK-1}^{\star})$ and not through $\theta_{cK}^{(1),\mathrm{post}}$.

  We now show that the pre-pass-$2$ midpoint residual equals $r^{(1)}$. Let $\hat{X}_{cK-1}^{(2)}$ denote the
  student activation at depth $cK{-}1$ presented to pass~$2$, and let $\theta_{cK-1}^{(2),\mathrm{pre}}$
  denote the pre-pass-$2$ state of block $cK{-}1$. It suffices to establish the two identities
  \begin{equation}
    \hat{X}_{cK-1}^{(2)} = \hat{X}_{cK-1}, ~~ \theta_{cK-1}^{(2),\mathrm{pre}} = \theta_{cK-1}^{(1),\mathrm{post}}.
    \label{eq:seam-invariance}
  \end{equation}

  \smallskip \noindent\emph{Prefix invariance.} By Lemma~\ref{lem:seam-structural}(b), pass~$1$ is the
  \emph{last} pair processed in chunk~$c$'s inner refinement (its index equals $w_{\max} - 1 = cK{-}1$), and
  pass~$2$ is the \emph{first} pair processed in chunk~$c{+}1$'s inner refinement (its index equals
  $w_{\min} = cK{-}1$ for the next chunk). We claim that every parameter $\theta_{b''}$ with
  $b'' \in \{0,\dots,cK{-}2\}$ takes the \emph{same} value at the start of pass~$1$ (when chunk~$c$'s pair
  loop begins processing pair $(cK{-}1,cK)$) and at the start of pass~$2$; denote this common value
  $\theta_{b''}^{(1)} \in \Theta_{b''}$. Indeed, pass~$1$ is a two-block CBQ subproblem on $(cK{-}1,cK)$, and
  its rollback-safeguarded Adam update modifies \emph{only} $\theta_{cK-1}$ and $\theta_{cK}$ (no other
  parameter vector enters the CBQ loss $\mathcal{L}_{\mathrm{cbq}}$ for this pair, and
  Algorithm~\ref{alg:icbq}'s advance step at line~5 is a read-only evaluation that does not update $\theta$).
  Hence $\theta_{b''}$ with $b'' \le cK{-}2$ is unchanged by pass~$1$, so its pre-pass-$1$ and end-of-pass-$1$
  values agree.

  Between the end of pass~$1$ and the start of pass~$2$, the outer driver executes exactly
  \begin{enumerate}[label=\textup{(\alph*)},topsep=2pt,itemsep=0pt]
    \item the end-of-chunk activation re-roll at the close of chunk~$c$;
    \item for each $b' \in \{cK,cK{+}1,\dots,\min\{(c{+}1)K{-}1,L{-}1\}\}$, the prefit
    step (when present), the inner quantizer $Q$ on $\theta_{b'}$, and a one-layer advance of the student and teacher streams through the updated block~$b'$;
    \item the in-pass re-roll that opens chunk~$c{+}1$'s inner refinement.
  \end{enumerate}
  Operations~(a) and~(c) recompute activations only, not parameters. Operation~(b) modifies $\theta_{b'}$ only
  for $b' \ge cK$. Hence every parameter $\theta_{b''}$ with $b'' \in \{0,\dots,cK{-}1\}$ is unchanged between
  the end of pass~$1$ and the start of pass~$2$. Combining this with the invariance under pass~$1$ for
  $b'' \le cK{-}2$,
  \begin{equation}
    \theta_{b''}^{\text{start of pass~}1} = \theta_{b''}^{\text{end of pass~}1} = \theta_{b''}^{\text{start of pass~}2} ~~\text{for }b'' \in \{0,\dots,cK{-}2\},
    \label{eq:prefix-invariance}
  \end{equation}
  which we take as the definition of $\theta_{b''}^{(1)}$; membership $\theta_{b''}^{(1)} \in \Theta_{b''}$
  follows from the construction of $\Theta_{b''}$ in Assumption~\ref{asm:cbq}. For $b'' = cK{-}1$,
  operations~(a)--(c) leave $\theta_{cK-1}$ unchanged, so
  $\theta_{cK-1}^{\text{end of pass~}1} = \theta_{cK-1}^{\text{start of pass~}2} = \theta_{cK-1}^{(1),\mathrm{post}} = \theta_{cK-1}^{(2),\mathrm{pre}}$
  by the very definition of these symbols.

  \smallskip \noindent\emph{Second identity in Eq.~\eqref{eq:seam-invariance}.} The chain
  $\theta_{cK-1}^{\text{end of pass~}1} = \theta_{cK-1}^{\text{start of pass~}2}$ just established provides
  $\theta_{cK-1}^{(2),\mathrm{pre}} = \theta_{cK-1}^{(1),\mathrm{post}}$; in particular
  $\theta_{cK-1}^{(2),\mathrm{pre}} \in \Theta_{cK-1}$.

  \smallskip \noindent\emph{First identity in Eq.~\eqref{eq:seam-invariance}.} For $q \in \{1,2\}$, let
  $\Phi^{(q)}$ denote the prefix map from depth $0$ to depth $cK{-}1$ using the parameters at the start of
  pass~$q$:
  \begin{equation}
    \Phi^{(q)} \coloneqq \begin{cases}
      \mathrm{Id}, & cK{-}1 = 0,\\[1mm]
      f_{cK-2}^{\theta_{cK-2}^{\text{start of pass~}q}} \circ \cdots \circ
      f_{0}^{\theta_{0}^{\text{start of pass~}q}}, & cK{-}1 \ge 1.
    \end{cases}
  \end{equation}
  By Assumption~\ref{asm:reroll}, the in-pass re-roll at the start of pass~$2$ produces
  $\hat{X}_{cK-1}^{(2)} = \Phi^{(2)}(X_{0})$ on the stored calibration inputs. Inside chunk~$c$'s
  inner-refinement pair loop, the same prefix map is applied incrementally before processing pair
  $(cK{-}1,cK)$, so $\hat{X}_{cK-1} = \Phi^{(1)}(X_{0})$. By Eq.~\eqref{eq:prefix-invariance}, the two prefix
  maps use \emph{identical} parameters $\theta_{b''}^{(1)}$ for every $b'' \in \{0,\dots,cK{-}2\}$; when this
  index set is empty, both maps are $\mathrm{Id}$. Hence $\hat{X}_{cK-1}^{(2)} = \hat{X}_{cK-1}$, which
  establishes the first identity. In particular, $\hat{X}_{cK-1}^{(2)} \in \mathcal{T}_{cK-1}$ since
  $\hat{X}_{cK-1} \in \mathcal{T}_{cK-1}$ by hypothesis.

  \smallskip \noindent\emph{Applying the contraction to pass~$2$.} We first record the pre-pass-$2$ state of
  block $cK$. Operation~(b) in the interval between pass~$1$ and pass~$2$ applies to block $cK$ (the $b' = cK$
  step) either the prefit step followed by the inner quantizer $Q$ (when $Q = \mathrm{DBF}$) or $Q$ alone
  (when $Q = \mathrm{GPTQ}$); in both cases the resulting post-inner-quantizer state lies in $\Theta_{cK}$ by
  the definition preceding Assumption~\ref{asm:cbq}. Subsequent operations in the interval (the in-pass
  re-roll~(c), the remaining prefit/inner-quantizer steps for
  $b' \in \{cK{+}1,\dots,\min\{(c{+}1)K{-}1,L{-}1\}\}$ if any, and the earlier end-of-chunk re-roll~(a)) do
  not touch $\theta_{cK}$. Hence $\theta_{cK}^{(2),\mathrm{pre}}$ equals this post-inner-quantizer state and
  lies in $\Theta_{cK}$. Combined with $\theta_{cK-1}^{(2),\mathrm{pre}} \in \Theta_{cK-1}$ and
  $\hat{X}_{cK-1}^{(2)} \in \mathcal{T}_{cK-1}$, this actual pass~$2$ is covered by the midpoint-contraction
  certificate~Eq.~\eqref{eq:asm-cbq}. The teacher activation $X_{cK}^{\star}$ is unchanged between the two
  passes because the teacher blocks are fixed by Assumption~\ref{asm:lip} and are reproduced verbatim by the
  teacher-side re-rolls of Assumption~\ref{asm:reroll}. Combining the two identities in
  Eq.~\eqref{eq:seam-invariance},
  \begin{equation}
    r^{(2),\mathrm{in}} \coloneqq \bigl\|f_{cK-1}^{\theta_{cK-1}^{(2),\mathrm{pre}}}(\hat{X}_{cK-1}^{(2)})-X_{cK}^{\star}\bigr\|_{F} = \bigl\|f_{cK-1}^{\theta_{cK-1}^{(1),\mathrm{post}}}(\hat{X}_{cK-1})-X_{cK}^{\star}\bigr\|_{F} = r^{(1)}.
  \end{equation}
  Crucially, $r^{(2),\mathrm{in}}$ depends on block $cK$ only through the teacher target $X_{cK}^{\star}$ and
  not through $\theta_{cK}^{(2),\mathrm{pre}}$; the transition
  $\theta_{cK}^{(1),\mathrm{post}} \mapsto \theta_{cK}^{(2),\mathrm{pre}}$ induced by the intervening
  prefit/inner-quantizer steps therefore does not enter the midpoint residual. This transition does enter the
  subsequent Adam trajectory of pass~$2$, since the aggregate CBQ loss $\mathcal{L}_{\mathrm{cbq}}$ depends on
  both block parameters; Assumption~\ref{asm:cbq} accommodates this by certifying the post-pass midpoint
  contraction for the actual admissible pre-pass pair
  $(\theta_{cK-1}^{(2),\mathrm{pre}},\theta_{cK}^{(2),\mathrm{pre}}) \in \Theta_{cK-1} \times \Theta_{cK}$
  produced by operation~(b). Applying Eq.~\eqref{eq:asm-cbq} to pass~$2$ with input residual
  $r^{(2),\mathrm{in}} = r^{(1)}$ and combining with Eq.~\eqref{eq:seam-pass1} yields
  \begin{equation}
    m_{cK-1}^{\mathrm{int,out}} \le \gamma\,r^{(2),\mathrm{in}}+\epsilon_{\mathrm{sub}} = \gamma\,r^{(1)}+\epsilon_{\mathrm{sub}} \le \gamma^{2}\,\widetilde{m}+(1+\gamma)\,\epsilon_{\mathrm{sub}},
    \label{eq:seam-int-proof}
  \end{equation}
  which is Eq.~\eqref{eq:seam-int}.
\end{proof}

\begin{remark}
  The ratio of ICBQ's post-seam residual to the sequential one is $\gamma$ in the leading term, at the cost of
  a single additional $\gamma\,\epsilon_{\mathrm{sub}}$ in the residual term. The leading-term gain is
  therefore visible whenever this additive term is small relative to the incoming seam residual.
\end{remark}

\subsection{Chain along the depth axis}
\label{app:chain}

We now chain Lemmas~\ref{lem:per-pair} and~\ref{lem:seam} along the full depth $0,\dots,L{-}1$. Let
$u_{i} \coloneqq m_{i}^{\mathrm{seq,out}}$ and $v_{i} \coloneqq m_{i}^{\mathrm{int,out}}$ denote the post-pair
midpoint residuals of the two schedules at pair $i = 0,\dots,L{-}2$. Write
\begin{equation}
  \tau \coloneqq \gamma\rho, ~~ R \coloneqq (1+\gamma)\,\epsilon_{\mathrm{sub}}, ~~ \widetilde{R} \coloneqq (1+\gamma+\gamma^{2})\,\epsilon_{\mathrm{sub}},
\end{equation}
and let $s(i) \coloneqq |\{c \in \{1,\dots,S\}:cK{-}1 \le i\}|$ be the number of chunk seams whose index is at
most $i$. Both schedules share the initial condition $u_{-1} = v_{-1} = 0$ (the calibration-sample-matched
inputs) and satisfy the recurrences in the following lemma; the remaining work of this subsection is to unroll
them.

\begin{lemma}[Per-pair recurrences]
  \label{lem:recur}
  For every $i \in \{0,\dots,L{-}2\}$,
  \begin{enumerate}[topsep=2pt,itemsep=2pt]
    \item[\textup{(S)}] $u_{i}\le \tau\,u_{i-1}+R$ uniformly over $i$
    \emph{(sequential CBQ)};
    \item[\textup{(I)}] $v_{i}\le \tau\,v_{i-1}+R$ if $i$ is a non-seam pair, and
    $v_{i}\le \gamma\tau\,v_{i-1}+\widetilde{R}$ if $i = cK{-}1$ is a seam pair \emph{(ICBQ)}.
  \end{enumerate}
\end{lemma}

\begin{proof}
  \emph{Statement~(S).} Sequential CBQ processes every pair exactly once, so Lemma~\ref{lem:per-pair} applied
  to pair $(i,i{+}1)$ yields
  $u_{i}\le \gamma\rho\,u_{i-1}+(1+\gamma)\,\epsilon_{\mathrm{sub}} = \tau\,u_{i-1}+R$.

  \smallskip \noindent\emph{Non-seam case of (I).} For a non-seam pair $i$ that receives a single CBQ pass,
  Lemma~\ref{lem:per-pair} gives $v_{i}\le\tau\,v_{i-1}+R$, identical to~(S).

  \smallskip \noindent\emph{Seam case of (I).} Let $i=cK{-}1$ be a seam pair. In the interleaved schedule, let
  $\hat{X}_{cK-1}$ denote the student activation at depth $cK{-}1$ presented to pass~$1$, and let
  $(\theta_{cK-1}^{\mathrm{pre}},\theta_{cK}^{\mathrm{pre}})$ denote the pre-pass-$1$ state of blocks $cK{-}1$
  and $cK$. Both components of this pair arise at an intermediate step of $\mathcal{S}_{\mathrm{int}}$:
  $\theta_{cK-1}^{\mathrm{pre}}$ is either the state left by the previous pair update in chunk~$c$'s inner
  refinement (when $cK{-}1 \ge 1$ and chunk $c$'s pair loop has processed at least one pair before the seam)
  or the post-inner-quantizer state of block $cK{-}1$ (when the seam is the first pair processed in chunk $c$,
  e.g.\ for $c = 1, K = 1$); both belong to $\Theta_{cK-1}$ by the definition preceding
  Assumption~\ref{asm:cbq}. The other component $\theta_{cK}^{\mathrm{pre}} = \theta_{cK}^{\star}$ is the fp16
  teacher copy (block $cK$ has not yet been processed by the prefit or inner-quantizer step at the start of
  pass~$1$), which belongs to $\Theta_{cK}$ by the same definition. Hence
  $(\theta_{cK-1}^{\mathrm{pre}},\theta_{cK}^{\mathrm{pre}}) \in \Theta_{cK-1} \times \Theta_{cK}$, verifying
  the hypothesis of Lemma~\ref{lem:seam}. By the construction of $\mathcal{T}_{cK-1}$ in
  Assumption~\ref{asm:lip}, $\hat{X}_{cK-1},X_{cK-1}^{\star} \in \mathcal{T}_{cK-1}$.

  We claim that $\|\hat{X}_{cK-1}-X_{cK-1}^{\star}\|_{F} = v_{i-1}$ (with the convention $v_{-1} = 0$) in both
  of the following cases:
  \begin{itemize}[topsep=2pt,itemsep=0pt]
    \item If $cK{-}1 \ge 1$ (which includes $c \ge 2$, and $c = 1$ with $K \ge 2$),
    chunk~$c$'s inner-refinement loop processes pair $(cK{-}2,cK{-}1)$ before pair $(cK{-}1,cK)$---the latter is the last pair in chunk~$c$'s window---so $\hat{X}_{cK-1}$ equals the post-pair-$(cK{-}2)$ student activation, and $\|\hat{X}_{cK-1}-X_{cK-1}^{\star}\|_{F} = m_{cK-2}^{\mathrm{int,out}} = v_{i-1}$ by Eq.~\eqref{eq:midpoint-def}.
    \item If $cK{-}1 = 0$ (which happens only for $c = 1, K = 1$), the seam pair $(0,1)$
    is the only pair in chunk~$1$'s window, and the in-pass re-roll (Assumption~\ref{asm:reroll}) gives $\hat{X}_{0} = X_{0}^{\star}$; hence $\|\hat{X}_{0}-X_{0}^{\star}\|_{F} = 0 = v_{-1}$ by the boundary convention.
  \end{itemize}
  In either case $\|\hat{X}_{cK-1}-X_{cK-1}^{\star}\|_{F} = v_{i-1}$.

  Let $\widetilde{m} =\|f_{cK-1}^{\theta_{cK-1}^{\mathrm{pre}}}(\hat{X}_{cK-1}) -X_{cK}^{\star}\|_{F}$ denote
  the pre-seam midpoint residual presented to pass~$1$ (in the notation of Lemma~\ref{lem:seam}). By the same
  triangle-inequality derivation as in Eq.~\eqref{eq:rin-bound} (local envelope in
  Eq.~\eqref{eq:local-envelope} applied with $\theta_{cK-1}=\theta_{cK-1}^{\mathrm{pre}},z=\hat{X}_{cK-1}$,
  and Assumption~\ref{asm:lip} applied with $x=\hat{X}_{cK-1},y=X_{cK-1}^{\star}$),
  \begin{equation}
    \widetilde{m} \le \rho \|\hat{X}_{cK-1}-X_{cK-1}^{\star}\|_{F} +\epsilon_{\mathrm{sub}} = \rho v_{i-1}+\epsilon_{\mathrm{sub}}.
    \label{eq:tildem-bound}
  \end{equation}
  Applying Lemma~\ref{lem:seam}(ii) (the two-pass seam contraction
  $v_{i}\le\gamma^{2}\widetilde{m}+(1+\gamma)\epsilon_{\mathrm{sub}}$) with the hypotheses just verified, and
  combining with Eq.~\eqref{eq:tildem-bound} and $\tau=\gamma\rho$,
  \begin{align*}
    v_{i} & \le \gamma^{2}\widetilde{m}+(1+\gamma)\epsilon_{\mathrm{sub}}\\
    & \le \gamma^{2}\bigl(\rho v_{i-1}+\epsilon_{\mathrm{sub}}\bigr) +(1+\gamma)\epsilon_{\mathrm{sub}}\\
    & = \gamma\rho\cdot\gamma v_{i-1}+(\gamma^{2}+1+\gamma)\epsilon_{\mathrm{sub}}\\
    & = \gamma\tau v_{i-1}+\widetilde{R},
  \end{align*}
  where the last equality uses $\widetilde{R}=(1+\gamma+\gamma^{2})\epsilon_{\mathrm{sub}}$.
\end{proof}

Unrolling Lemma~\ref{lem:recur}(S) from $u_{-1} = 0$ yields a closed-form sequential upper bound
$\mathcal{U}_{i}$; unrolling Lemma~\ref{lem:recur}(I) provides a closed-form interleaved upper bound
$\mathcal{V}_{i}$ in which a contribution originating from pair $k$ is multiplied by an additional factor
$\gamma$ for each subsequent seam.

\begin{proposition}[Pair-level unrolling]
  \label{prop:chain}
  Under Assumptions~\ref{asm:lip}, \ref{asm:cbq}, and~\ref{asm:reroll} and $\tau<1$, for every
  $i \in \{0,\dots,L{-}2\}$ the midpoint residuals satisfy $u_{i} \le \mathcal{U}_{i}$ and
  $v_{i} \le \mathcal{V}_{i}$ with the deterministic upper bounds
  \begin{align}
    \mathcal{U}_{i} &\coloneqq R \sum_{k=0}^{i}\tau^{i-k} = R \frac{1-\tau^{i+1}}{1-\tau}, \label{eq:unroll-seq}\\[2pt]
    \mathcal{V}_{i} &\coloneqq \widetilde{R} \sum_{k=0}^{i}\tau^{i-k}\gamma^{s(i)-s(k)} \le \widetilde{R} \frac{1-\tau^{i+1}}{1-\tau}. \label{eq:unroll-int}
  \end{align}
  Moreover the two bounds admit the direct comparison
  \begin{equation}
    \mathcal{V}_{i} \le \gamma^{ s(i)} \mathcal{U}_{i} + \bigl(\widetilde{R}-\gamma^{ s(i)}R\bigr) \frac{1-\tau^{ i+1}}{1-\tau},
    \label{eq:unroll-compare}
  \end{equation}
  where the factor $\gamma^{ s(i)}$ quantifies the cumulative seam-amortized gain in the propagated term and
  $(\widetilde{R}-\gamma^{ s(i)}R)\ge 0$ upper-bounds the additive terms not carrying this full factor.
\end{proposition}

\begin{proof}
  \emph{Proof of Eq.~\eqref{eq:unroll-seq}.} From $u_{-1} = 0$, Lemma~\ref{lem:recur}(S) gives, by a one-step
  induction, the base $u_{0} \le \tau\cdot0+R = R = \mathcal{U}_{0}$ and the step
  $u_{i} \le \tau u_{i-1}+R \le \tau \mathcal{U}_{i-1}+R =R\bigl(\tau\sum_{k=0}^{i-1}\tau^{i-1-k}+1\bigr) =R\sum_{k=0}^{i}\tau^{i-k}=\mathcal{U}_{i}$.

  \emph{Proof of Eq.~\eqref{eq:unroll-int}.} Define, for $\ell \in \{0,\dots,L{-}2\}$,
  \begin{equation}
    \alpha_{\ell} \coloneqq \begin{cases}\tau & \text{if }\ell\text{ is non-seam},\\
    \gamma\tau & \text{if }\ell=cK{-}1\text{ is seam},\end{cases}
    ~~
    \beta_{\ell} \coloneqq \begin{cases}R & \text{if }\ell\text{ is non-seam},\\
    \widetilde{R} & \text{if }\ell=cK{-}1\text{ is seam}.\end{cases}
  \end{equation}
  Both $\alpha_{\ell},\beta_{\ell}$ are nonnegative, and $0 \le \beta_{\ell} \le \widetilde{R}$ (since
  $\widetilde{R} = R+\gamma^{2}\epsilon_{\mathrm{sub}} \ge R$). With the convention that an empty product is
  $1$, induction from $v_{-1} = 0$ using Lemma~\ref{lem:recur}(I) gives
  \begin{equation}
    v_{i} \le \sum_{k=0}^{i}\beta_{k} \prod_{\ell=k+1}^{i}\alpha_{\ell} \le \widetilde{R}\sum_{k=0}^{i} \prod_{\ell=k+1}^{i}\alpha_{\ell},
    \label{eq:v-first-bound}
  \end{equation}
  because the base case is $v_{0}\le\alpha_{0}v_{-1}+\beta_{0}=\beta_{0}$, and the induction step follows from
  \begin{equation}
    v_i\le\alpha_i v_{i-1}+\beta_i \le \sum_{k=0}^{i-1}\beta_k \prod_{\ell=k+1}^{i}\alpha_{\ell} + \beta_i.
  \end{equation}
  The second inequality in Eq.~\eqref{eq:v-first-bound} uses $\beta_{k} \le \widetilde{R}$ and the
  nonnegativity of the product. The product $\prod_{\ell=k+1}^{i}\alpha_{\ell}$ picks up a factor $\tau$ at
  every $\ell \in \{k{+}1,\dots,i\}$ and an \emph{extra} factor $\gamma$ at every seam in that range. The
  number of seams in $\{k{+}1,\dots,i\}$ is $s(i) - s(k)$ by definition of $s(\cdot)$, so
  \begin{equation}
    \prod_{\ell=k+1}^{i}\alpha_{\ell} = \tau^{ i-k} \gamma^{ s(i)-s(k)}.
    \label{eq:prod-factor}
  \end{equation}
  The product identity therefore gives the interleaved closed form
  \begin{equation*}
    v_{i} \le \mathcal{V}_{i} = \widetilde{R}\sum_{k=0}^{i}\tau^{i-k}\gamma^{s(i)-s(k)}.
  \end{equation*}
  Its geometric upper bound follows from $\gamma^{s(i)-s(k)} \le 1$, valid because $s(i) \ge s(k)$ and
  $\gamma \in (0,1)$.

  \emph{Direct comparison.} The quantities $\mathcal{U}_{i},\mathcal{V}_{i}$ are deterministic once
  $\tau,\gamma,R,\widetilde{R}$ are fixed. Comparing their closed forms and pulling the geometric factor
  $\gamma^{s(i)}$ out of $\mathcal{U}_{i}$,
  \begin{equation}
    \mathcal{V}_{i}-\gamma^{ s(i)} \mathcal{U}_{i} = \sum_{k=0}^{i}\tau^{ i-k}\Bigl[\widetilde{R} \gamma^{ s(i)-s(k)}-\gamma^{ s(i)} R\Bigr].
    \label{eq:V-minus-gammaU}
  \end{equation}
  For each term, $\gamma^{ s(i)-s(k)} \le 1$ and $\widetilde{R}\ge 0$ imply
  $\widetilde{R} \gamma^{ s(i)-s(k)} \le \widetilde{R}$, so
  \begin{equation}
    \mathcal{V}_{i}-\gamma^{ s(i)} \mathcal{U}_{i} \le \sum_{k=0}^{i}\tau^{ i-k}\bigl(\widetilde{R}-\gamma^{ s(i)}R\bigr) = \bigl(\widetilde{R}-\gamma^{ s(i)}R\bigr) \frac{1-\tau^{ i+1}}{1-\tau}.
    \label{eq:V-minus-gammaU-bound}
  \end{equation}
  Rearranging yields the stated comparison. The coefficient $\widetilde{R}-\gamma^{ s(i)}R$ is nonnegative:
  since $\gamma \in (0,1)$ and $s(i) \ge 0$, $\gamma^{ s(i)} \le 1$, and thus
  $\gamma^{ s(i)}R \le R \le \widetilde{R}$. Hence the comparison is a nontrivial upper bound of
  $\mathcal{V}_{i}$, and the two components have the claimed interpretation: $\gamma^{ s(i)}\mathcal{U}_{i}$
  is the propagated contribution amortized by $s(i)$ seams, and
  $(\widetilde{R}-\gamma^{ s(i)}R)(1-\tau^{ i+1})/(1-\tau)$ collects the additive terms that do not carry the
  full $\gamma^{ s(i)}$ factor.
\end{proof}

The bound in Eq.~\eqref{eq:unroll-compare} is convenient for matching the sequential envelope
$\mathcal{U}_{i}$, but it overestimates the residual contribution of non-seam pairs by replacing each
$\beta_{k}$ by the seam-only upper bound $\widetilde{R}$. The next proposition records a seam-aware refinement
that splits the contribution of seam and non-seam pairs explicitly; it is used only to justify
Remark~\ref{rem:C-nontight} and is not invoked elsewhere in the proof.

\begin{proposition}[Seam-aware sharper comparison]
  \label{prop:chain-sharp}
  Let $\mathcal{J}_{i} \coloneqq \{k\in\{0,\dots,i\}:k = cK{-}1\text{ for some }c\in\{1,\dots,S\}\}$ be the
  set of seam indices up to pair $i$. Under the hypotheses of Proposition~\ref{prop:chain},
  \begin{equation}
    v_{i} \le \sum_{k=0}^{i} \bigl(R + \gamma^{2}\epsilon_{\mathrm{sub}} \mathbf{1}\{k\in\mathcal{J}_{i}\}\bigr) \tau^{i-k} \gamma^{s(i)-s(k)}.
    \label{eq:sharp-v-unroll}
  \end{equation}
  Consequently,
  \begin{equation}
    v_{i} \le \gamma^{s(i)} \mathcal{U}_{i} + D_{i} \epsilon_{\mathrm{sub}},
    \label{eq:sharp-v-compare}
  \end{equation}
  where
  \begin{equation}
    D_{i} \coloneqq (1+\gamma) \sum_{k=0}^{i} \tau^{i-k}\bigl(\gamma^{s(i)-s(k)}-\gamma^{s(i)}\bigr) + \gamma^{2} \sum_{k\in\mathcal{J}_{i}} \tau^{i-k} \gamma^{s(i)-s(k)}.
    \label{eq:sharp-Di}
  \end{equation}
  In particular, when $S(L,K) = 0$ (single chunk, no seams), $\mathcal{J}_{i} = \emptyset$ and
  $s(i) = s(k) = 0$ for every $k\le i$, so $D_{i} = 0$ and $v_{i} \le \mathcal{U}_{i}$ exactly.
\end{proposition}

\begin{proof}
  The interleaved forcing coefficient is $\beta_{k} = R$ at non-seam pairs and
  $\beta_{k} = \widetilde{R} = R + \gamma^{2}\epsilon_{\mathrm{sub}}$ at seam pairs, that is,
  $\beta_{k} = R + \gamma^{2}\epsilon_{\mathrm{sub}}\mathbf{1}\{k\in\mathcal{J}_{i}\}$. Keeping this
  seam-dependent forcing in the exact variation-of-constants expansion
  \begin{equation*}
    v_{i}\le\sum_{k=0}^{i}\beta_{k}\prod_{\ell=k+1}^{i}\alpha_{\ell}
  \end{equation*}
  (without invoking $\beta_{k}\le\widetilde{R}$) and using the same product identity gives the seam-aware
  closed form. To extract the comparison with the sequential envelope, write
  \begin{equation*}
    v_{i}-\gamma^{s(i)}\mathcal{U}_{i} = R\sum_{k=0}^{i}\tau^{i-k}\bigl(\gamma^{s(i)-s(k)}-\gamma^{s(i)}\bigr) + \gamma^{2}\epsilon_{\mathrm{sub}}\sum_{k\in\mathcal{J}_{i}}\tau^{i-k}\gamma^{s(i)-s(k)},
  \end{equation*}
  using $R = (1+\gamma)\epsilon_{\mathrm{sub}}$ and rearranging. Both summands are nonnegative because
  $s(i)\ge s(k)$ implies $\gamma^{s(i)-s(k)}\ge \gamma^{s(i)}$ (as $\gamma\in(0,1)$). When $S = 0$,
  $\mathcal{J}_{i}=\emptyset$ trivially zeros the second sum, and the first sum vanishes term by term because
  $s(i)-s(k) = 0$. Hence $D_{i} = 0$ and the bound becomes
  $v_{i}\le \gamma^{s(i)}\mathcal{U}_{i} = \mathcal{U}_{i}$.
\end{proof}

\begin{corollary}[Last-pair upper bounds]
  \label{cor:last-pair}
  At $i = L{-}2$ (so $s(L{-}2) = S$), $u_{L-2} \le \mathcal{U}_{L-2}$ and $v_{L-2} \le \mathcal{V}_{L-2}$,
  with
  \begin{align}
    \mathcal{U}_{L-2}& = R \frac{1-\tau^{L-1}}{1-\tau}, \label{eq:last-seq}\\
    \mathcal{V}_{L-2}& \le \gamma^{S} \mathcal{U}_{L-2} + \bigl(\widetilde{R}-\gamma^{S}R\bigr) \frac{1-\tau^{L-1}}{1-\tau}. \label{eq:last-int}
  \end{align}
\end{corollary}

\begin{proof}
  Apply Proposition~\ref{prop:chain} to the last valid pair $i = L{-}2$. It remains only to check that
  $s(L{-}2) = S$. Recall that the seam count is $S = \lceil L/K\rceil-1$ and that the seams are indexed by
  $c \in \{1,\dots,S\}$ with seam index $cK{-}1$. If $S = 0$, there are no seams and $s(L{-}2) = 0 = S$.
  Assume henceforth $S \ge 1$. Write $L = qK+r$ with $q = \lfloor L/K\rfloor$ and $0 \le r < K$. Two cases:
  \begin{itemize}
    \item If $r = 0$, then $\lceil L/K\rceil = q$, so
    $S = q{-}1 \ge 1$, i.e.\ $q \ge 2$, and
    $SK = (q{-}1)K = L{-}K \le L{-}1$ because
    $K \ge 1$.
    \item If $r \ge 1$, then $\lceil L/K\rceil = q{+}1$, so
    $S = q \ge 1$ and
    $SK = qK = L{-}r \le L{-}1$ because $r \ge 1$.
  \end{itemize}
  In both cases $SK \le L{-}1$. Hence for every $c \in \{1,\dots,S\}$, the seam index satisfies
  $cK{-}1 \le SK{-}1 \le L{-}2$, and is therefore counted by $s(L{-}2)$. Conversely, by the definition of $S$,
  there are no seams with $c > S$, so no seam index exceeds $SK{-}1 \le L{-}2$. We conclude $s(L{-}2) = S$.
\end{proof}

\subsection{\texorpdfstring{From last-pair residual to depth-$L$ mismatch}{From last-pair residual to depth-L mismatch}}
\label{app:finale}

We propagate the last-pair midpoint residual through the final block $L{-}1$ in two formal steps: a
monotonicity-of-final-visit-times lemma (Lemma~\ref{lem:T-monotone}) and a final-prefix-invariance lemma
(Lemma~\ref{lem:final-prefix}). For every $b \in \{0,\dots,L{-}2\}$, let $T_{b}$ denote the algorithmic time
of the \emph{last} CBQ pass on pair $(b,b{+}1)$: in Sequential CBQ this is the unique pass on that pair; in
ICBQ it is the unique pass if $(b,b{+}1)$ is a non-seam pair, and the second pass if $(b,b{+}1)$ is a seam
pair (i.e.\ $b = cK{-}1$ for some $c \in \{1,\dots,S\}$). We extend by $T_{L-1} \coloneqq T_{L-2}$, so that
block $L{-}1$'s final state is pinned at $T_{L-2}$. For each $b \in \{0,\dots,L{-}1\}$, write
$\theta_{b}^{\mathrm{post}}$ for the parameter state of block $b$ at time $T_{b}$.

\begin{lemma}[Monotonicity of final-visit times and block finalization]
  \label{lem:T-monotone}
  Under the schedules described by Algorithms~\ref{alg:icbq-driver}--\ref{alg:icbq},
  \begin{enumerate}[label=\textup{(\alph*)},topsep=2pt,itemsep=0pt]
    \item the times satisfy
    $T_{0} < T_{1} < \cdots < T_{L-2} = T_{L-1}$;
    \item for every $b \in \{0,\dots,L{-}1\}$, no operation that modifies $\theta_{b}$
    occurs strictly after $T_{b}$; consequently, block $b$'s state at every time
    $t \ge T_{b}$ equals $\theta_{b}^{\mathrm{post}}$, and
    $\theta_{b}^{\mathrm{post}} \in \Theta_{b}$ by Assumption~\ref{asm:cbq}.
  \end{enumerate}
\end{lemma}

\begin{proof}
  \emph{(a)} The outer driver iterates $b = 0,1,\dots,L{-}1$ in increasing order, and the inner pair loop
  processes pairs in increasing order of the first-block index within each chunk
  (Algorithms~\ref{alg:icbq-driver}--\ref{alg:icbq}). For non-seam pairs, the unique CBQ pass occurs inside
  the chunk-closing inner refinement at some outer iteration $b'$, and its algorithmic time is strictly
  increasing in the pair index $b$ (later pairs are processed later inside the same pair loop, and later
  chunks close later than earlier chunks). For a seam pair $(cK{-}1,cK)$, the second pass occurs at chunk
  $c{+}1$'s closure, which is strictly later than chunk $c$'s closure (where the first pass occurs); the
  second pass on the seam is itself the first pair processed in chunk $c{+}1$'s pair loop, so its algorithmic
  time precedes every other CBQ pass in chunk $c{+}1$. Combining these observations,
  $T_{0} < T_{1} < \cdots < T_{L-2}$. The convention $T_{L-1} = T_{L-2}$ is by definition.

  \emph{(b)} The operations that may modify $\theta_{b}$ are
  \begin{enumerate}[label=\textup{(A\arabic*)},topsep=2pt,itemsep=0pt]
    \item the outer-driver prefit/inner-quantizer step on block $b$, which occurs exactly
    once at outer iteration $b$; call its algorithmic time $Q_{b}$;
    \item CBQ updates arising from pairs that include $b$, namely $(b{-}1,b)$ (defined
    only for $b \ge 1$) and $(b,b{+}1)$ (defined only for $b \le L{-}2$); re-rolls
    recompute activations only and modify no parameter.
  \end{enumerate}

  For type~(A2): the last CBQ pass on $(b,b{+}1)$ is $T_{b}$ by definition (for $b \le L{-}2$); the last CBQ
  pass on $(b{-}1,b)$ is $T_{b-1}$ (for $b \ge 1$). By~(a), $T_{b-1} < T_{b}$ for $1 \le b \le L{-}2$, and
  $T_{L-2} = T_{L-1}$ for $b = L{-}1$. Hence no type-(A2) operation occurs strictly after $T_{b}$.

  For type~(A1): $Q_{b}$ occurs at the end of outer iteration $b$, which strictly precedes every chunk-closing
  inner refinement triggered at outer iteration $b' \ge b$. If $b \le L{-}2$, the first CBQ pass on pair
  $(b,b{+}1)$ occurs inside the chunk closing at some $b' \ge b$, hence strictly after $Q_{b}$; in particular
  $Q_{b} < T_{b}$. If $b = L{-}1$, then $T_{L-1} = T_{L-2}$ lies within the chunk closing at outer iteration
  $L{-}1$ (the final chunk), which is strictly after $Q_{L-1}$, the quantizer step of outer iteration $L{-}1$;
  hence $Q_{L-1} < T_{L-1}$.

  Combining the two cases, no operation modifying $\theta_{b}$ occurs strictly after $T_{b}$, so block $b$'s
  state at every $t \ge T_{b}$ equals $\theta_{b}^{\mathrm{post}}$. Admissibility
  $\theta_{b}^{\mathrm{post}} \in \Theta_{b}$ follows from the definition preceding Assumption~\ref{asm:cbq}.
\end{proof}

\begin{lemma}[Final-prefix identity for $\hat{X}_{L-1}$]
  \label{lem:final-prefix}
  Under Assumptions~\ref{asm:lip}, \ref{asm:cbq}, and~\ref{asm:reroll} and the schedules described in
  Algorithms~\ref{alg:icbq-driver}--\ref{alg:icbq}, for every $b \in \{0,\dots,L{-}2\}$ the student activation
  at depth $b{+}1$ at time $T_{b}$ (immediately after the advance step following that CBQ pass) satisfies
  \begin{equation}
    \hat{X}_{b+1}^{[T_{b}]} = \bigl(f_{b}^{\theta_{b}^{\mathrm{post}}} \circ \cdots \circ f_{0}^{\theta_{0}^{\mathrm{post}}}\bigr)(X_{0}).
    \label{eq:final-prefix-inv}
  \end{equation}
  In particular,
  $\hat{X}_{L-1} = (f_{L-2}^{\theta_{L-2}^{\mathrm{post}}} \circ \cdots \circ f_{0}^{\theta_{0}^{\mathrm{post}}})(X_{0})$,
  and this value is preserved by any subsequent activation-only re-roll (no parameter is modified after
  $T_{L-1}$).
\end{lemma}

\begin{proof}
  We argue by induction on $b \in \{0,\dots,L{-}2\}$.

  \emph{Base case $b = 0$.}
  By Assumption~\ref{asm:reroll} applied with $d = 0$ (empty composition), $\hat{X}_{0} = X_{0}$, and the
  advance at $T_{0}$ gives
  $\hat{X}_{1}^{[T_{0}]} = f_{0}^{\theta_{0}^{\mathrm{post}}}(\hat{X}_{0}) = f_{0}^{\theta_{0}^{\mathrm{post}}}(X_{0})$,
  which is Eq.~\eqref{eq:final-prefix-inv} at $b = 0$.

  \emph{Inductive step.}
  Assume Eq.~\eqref{eq:final-prefix-inv} holds at $T_{b-1}$ (with $b \ge 1$). In the interval
  $(T_{b-1},T_{b}]$ the algorithm executes some combination of outer-driver block-level steps on indices
  $b' \ge b$ (the blocks whose outer-driver quantization has not yet taken place by time $T_{b-1}$),
  end-of-chunk and in-pass re-rolls, and CBQ passes on pairs $(i,i{+}1)$ with $i \ge b$ (since the last CBQ
  pass on any pair $(i,i{+}1)$ with $i < b$ is $T_{i} \le T_{b-1}$, the right endpoint of the preceding
  interval). By Lemma~\ref{lem:T-monotone}, $\theta_{b'}$ for every $b' \le b{-}1$ remains pinned at its final
  state $\theta_{b'}^{\mathrm{post}}$ throughout this interval.

  We now track $\hat{X}_{b}$ during $(T_{b-1},T_{b}]$. The only operations that can modify $\hat{X}_{b}$ are:
  \begin{enumerate}[label=\textup{(\roman*)},topsep=2pt,itemsep=0pt]
    \item re-rolls (which recompute the student prefix up to some depth $\ge b$);
    \item the advance step at line~5 of Algorithm~\ref{alg:icbq} following a CBQ pass on
    pair $(i,i{+}1)$ with $i{+}1 = b$, i.e.\ pair $(b{-}1,b)$.
  \end{enumerate}
  By Lemma~\ref{lem:T-monotone}, no CBQ pass on pair $(b{-}1,b)$ occurs strictly after $T_{b-1}$, so advance
  steps of type~(ii) do not occur in the open interval $(T_{b-1},T_{b}]$. Hence $\hat{X}_{b}$ is modified only
  by re-rolls in this interval. By Assumption~\ref{asm:reroll}, any re-roll in the interval produces
  $\hat{X}_{b} = (f_{b-1}^{\theta_{b-1}^{\mathrm{current}}} \circ \cdots \circ f_{0}^{\theta_{0}^{\mathrm{current}}})(X_{0})$;
  since $\theta_{b'}^{\mathrm{current}} = \theta_{b'}^{\mathrm{post}}$ for every $b' \le b{-}1$ throughout the
  interval, this equals the \emph{final-prefix} value
  $P_{b} \coloneqq (f_{b-1}^{\theta_{b-1}^{\mathrm{post}}} \circ \cdots \circ f_{0}^{\theta_{0}^{\mathrm{post}}})(X_{0})$.
  The inductive hypothesis at depth $b{-}1$ states $\hat{X}_{b}^{[T_{b-1}]} = P_{b}$, so at the left endpoint
  of the interval $\hat{X}_{b}$ is already at the final-prefix value; between re-rolls, $\hat{X}_{b}$ is
  stored and untouched by any other operation, hence frozen at $P_{b}$; every re-roll resets it to $P_{b}$.
  Consequently, $\hat{X}_{b}^{[t]} = P_{b}$ for every $t \in (T_{b-1},T_{b}]$.

  At time $T_{b}$, the CBQ pass on pair $(b,b{+}1)$ sees this $\hat{X}_{b}$ as its input; immediately after
  the pass, the advance gives
  $\hat{X}_{b+1}^{[T_{b}]} = f_{b}^{\theta_{b}^{\mathrm{post}}}(\hat{X}_{b}) = (f_{b}^{\theta_{b}^{\mathrm{post}}} \circ \cdots \circ f_{0}^{\theta_{0}^{\mathrm{post}}})(X_{0})$,
  which is Eq.~\eqref{eq:final-prefix-inv} at $b$.

  Specialising $b = L{-}2$ yields
  $\hat{X}_{L-1} = (f_{L-2}^{\theta_{L-2}^{\mathrm{post}}} \circ \cdots \circ f_{0}^{\theta_{0}^{\mathrm{post}}})(X_{0})$.
  By Lemma~\ref{lem:T-monotone}(b), no parameter is modified after $T_{L-1} = T_{L-2}$; any later
  activation-only re-roll therefore recomputes $\hat{X}_{L-1}$ to the same final-prefix value, leaving it
  unchanged.
\end{proof}

\begin{proof}[Proof of Theorem~\ref{thm:seam-formal}]
  \emph{Pathwise depth-$L$ bound.}
  By Lemma~\ref{lem:final-prefix},
  \begin{equation}
    \hat{X}_{L-1} = \bigl(f_{L-2}^{\theta_{L-2}^{\mathrm{post}}} \circ \cdots \circ f_{0}^{\theta_{0}^{\mathrm{post}}}\bigr)(X_{0})
    \label{eq:final-prefix-XLm1}
  \end{equation}
  holds at the moment immediately after the CBQ pass at time $T_{L-2}$, and it is preserved by any subsequent
  activation-only re-roll. In particular, after the final chunk closes, the end-of-chunk re-roll (or
  equivalently an explicit final forward pass through the frozen final block) computes
  $\hat{X}_{L} = f_{L-1}^{\theta_{L-1}^{\mathrm{post}}}(\hat{X}_{L-1})$. This is an exact application of the
  final block map by Assumption~\ref{asm:reroll}; no later operation modifies $\theta_{L-1}$ by
  Lemma~\ref{lem:T-monotone}(b) with $b = L{-}1$.

  Hence $\hat{X}_{L-1} \in \mathcal{T}_{L-1}$ by Assumption~\ref{asm:lip},
  $X_{L-1}^{\star} \in \mathcal{T}_{L-1}$ by construction, and
  $\|\hat{X}_{L-1}-X_{L-1}^{\star}\|_{F} = m_{L-2}^{\mathrm{out}}$ by definition in
  Eq.~\eqref{eq:midpoint-def}. By the triangle inequality followed by the local envelope in
  Eq.~\eqref{eq:local-envelope} (applied with $\theta_{L-1} = \theta_{L-1}^{\mathrm{post}} \in \Theta_{L-1}$
  and $z = \hat{X}_{L-1} \in \mathcal{T}_{L-1}$) and Assumption~\ref{asm:lip},
  \begin{equation}
    \begin{aligned}
      \|e_{L}\|_{F} & = \|f_{L-1}^{\theta_{L-1}^{\mathrm{post}}}(\hat{X}_{L-1})-f_{L-1}^{\star}(X_{L-1}^{\star})\|_{F}\\
      & \le \|f_{L-1}^{\theta_{L-1}^{\mathrm{post}}}(\hat{X}_{L-1})-f_{L-1}^{\star}(\hat{X}_{L-1})\|_{F} +\|f_{L-1}^{\star}(\hat{X}_{L-1})-f_{L-1}^{\star}(X_{L-1}^{\star})\|_{F}\\
      & \le \epsilon_{\mathrm{sub}}+\rho \|\hat{X}_{L-1}-X_{L-1}^{\star}\|_{F} = \rho m_{L-2}^{\mathrm{out}}+\epsilon_{\mathrm{sub}}.
    \end{aligned}
    \label{eq:last-layer}
  \end{equation}
  This inequality is pathwise for each calibration sample.

  \emph{Sequential bound.}
  For the sequential schedule, apply Corollary~\ref{cor:last-pair}: the pathwise bound
  $u_{L-2}^{\mathrm{out},(n)} \le \mathcal{U}_{L-2}$ holds for every calibration sample $n \in \{1,\dots,N\}$;
  the right-hand side $\mathcal{U}_{L-2}$ is sample-independent, so this is a uniform pathwise bound. The
  last-layer Lipschitz inequality then gives, pathwise on each sample,
  $\|e_{L}^{\mathrm{seq},(n)}\|_{F} \le \rho \mathcal{U}_{L-2}+\epsilon_{\mathrm{sub}}$ for every $n$. Taking
  the calibration expectation $\mathbb{E}[\,\cdot\,] = \tfrac{1}{N}\sum_{n=1}^{N}(\cdot)^{(n)}$ preserves the
  inequality because the right-hand side is sample-independent, and the sum
  $\rho \mathcal{U}_{L-2}+\epsilon_{\mathrm{sub}}$ equals $\mathcal{B}_{L}^{\mathrm{seq}}$ from
  Eq.~\eqref{eq:B-seq-def}:
  \begin{equation}
    \mathbb{E}\|e_{L}^{\mathrm{seq}}\|_{F} \le \rho \mathcal{U}_{L-2}+\epsilon_{\mathrm{sub}} = \mathcal{B}_{L}^{\mathrm{seq}}.
    \label{eq:B-seq}
  \end{equation}

  \emph{Interleaved bound.}
  For ICBQ, apply Eq.~\eqref{eq:last-layer} pathwise to each sample $n$, then use the pathwise bound
  $v_{L-2}^{\mathrm{out},(n)} \le \mathcal{V}_{L-2}$ from Corollary~\ref{cor:last-pair} (with
  $\mathcal{V}_{L-2}$ sample-independent), and finally apply the deterministic bound-level comparison in
  Eq.~\eqref{eq:last-int}:
  \begin{equation*}
    \|e_{L}^{\mathrm{int},(n)}\|_{F} \le \rho\, v_{L-2}^{\mathrm{out},(n)}+\epsilon_{\mathrm{sub}} \le \rho\, \mathcal{V}_{L-2}+\epsilon_{\mathrm{sub}} \le \rho \left[\gamma^{S}\mathcal{U}_{L-2} +\bigl(\widetilde{R}-\gamma^{S}R\bigr)\frac{1-\tau^{L-1}}{1-\tau}\right] +\epsilon_{\mathrm{sub}}.
  \end{equation*}
  Using the identity
  $\epsilon_{\mathrm{sub}} = \gamma^{S}\epsilon_{\mathrm{sub}} +(1-\gamma^{S})\epsilon_{\mathrm{sub}}$,
  reshuffling terms by purely algebraic rearrangement, and taking calibration expectations (all right-hand
  sides are sample-independent, so the inequality is preserved),
  \begin{align}
    \mathbb{E}\|e_{L}^{\mathrm{int}}\|_{F} & \le \rho\, \mathcal{V}_{L-2}+\epsilon_{\mathrm{sub}}\nonumber\\
    & \le \rho\Bigl[\gamma^{S} \mathcal{U}_{L-2}+\bigl(\widetilde{R}-\gamma^{S}R\bigr)\frac{1-\tau^{L-1}}{1-\tau}\Bigr]+\epsilon_{\mathrm{sub}}\nonumber\\
    & = \gamma^{S}\bigl[\rho\, \mathcal{U}_{L-2}+\epsilon_{\mathrm{sub}}\bigr] + \bigl(\widetilde{R}-\gamma^{S}R\bigr)\frac{\rho\, (1-\tau^{L-1})}{1-\tau} + (1-\gamma^{S})\,\epsilon_{\mathrm{sub}}\nonumber\\
    & = \gamma^{S}\, \mathcal{B}_{L}^{\mathrm{seq}} + C(L,K,\rho,\gamma)\,\epsilon_{\mathrm{sub}}, \label{eq:finale}
  \end{align}
  where the last equality uses the definition in Eq.~\eqref{eq:B-seq-def} of $\mathcal{B}_{L}^{\mathrm{seq}}$
  together with the identification
  \begin{equation}
    C(L,K,\rho,\gamma)\,\epsilon_{\mathrm{sub}} = \bigl(\widetilde{R}-\gamma^{S}R\bigr)\frac{\rho\,(1-\tau^{L-1})}{1-\tau} + (1-\gamma^{S})\,\epsilon_{\mathrm{sub}}.
  \end{equation}
  Expanding the residual coefficient
  $\widetilde{R}-\gamma^{S}R =\bigl[(1+\gamma+\gamma^{2})-\gamma^{S}(1+\gamma)\bigr]\epsilon_{\mathrm{sub}}$
  yields the admissible residual constant whenever $\epsilon_{\mathrm{sub}} > 0$. If
  $\epsilon_{\mathrm{sub}} = 0$, then $R=\widetilde{R}=0$ and $\mathcal{B}_{L}^{\mathrm{seq}} = 0$, so the
  preceding pathwise recurrences give $\mathbb{E}\|e_{L}^{\mathrm{int}}\|_{F}=0$; choosing the same displayed
  value of $C(L,K,\rho,\gamma)$ leaves $C(L,K,\rho,\gamma)\,\epsilon_{\mathrm{sub}}=0$. Hence
  Eq.~\eqref{eq:thm-seam-C-admissible} is valid in all cases.

  \emph{Depth-uniform residual constant.}
  Using $\gamma^{S} \ge 0$,
  \begin{equation}
    \widetilde{R}-\gamma^{S}R = \bigl[(1+\gamma+\gamma^{2})-\gamma^{S}(1+\gamma)\bigr]\epsilon_{\mathrm{sub}} \le (1+\gamma+\gamma^{2})\,\epsilon_{\mathrm{sub}},
  \end{equation}
  together with $(1-\gamma^{S}) \le 1$ and $(1-\tau^{L-1})/(1-\tau) \le 1/(1-\tau)$ (from $\tau < 1$), we
  obtain the depth-uniform bound
  \begin{equation}
    C(L,K,\rho,\gamma) \le (1-\gamma^{S}) + (1+\gamma+\gamma^{2}) \frac{\rho (1-\tau^{L-1})}{1-\tau} \le 1 + \frac{(1+\gamma+\gamma^{2}) \rho}{1-\tau},
    \label{eq:C-explicit}
  \end{equation}
  independent of $L$ and $K$. Under the standard stability regime $\rho \le 1$ this simplifies to
  $C \le 1+(1+\gamma+\gamma^{2})/(1-\tau)$, as stated in Eq.~\eqref{eq:thm-seam-C}; the near-divergent case
  $\rho > 1$ is handled in Appendix~\ref{app:stability} under the same sufficient condition $\rho < 1/\gamma$
  (equivalently $\tau < 1$). Eq.~\eqref{eq:finale} is exactly the statement of Theorem~\ref{thm:seam-formal}.
\end{proof}

\begin{remark}[On the interpretation of the upper bound]
  \label{rem:ub-proof}
  The derivation above produces the ICBQ bound
  $\mathcal{B}_{L}^{\mathrm{int}} \coloneqq \gamma^{S}\mathcal{B}_{L}^{\mathrm{seq}}+C\epsilon_{\mathrm{sub}}$
  and the sequential bound $\mathcal{B}_{L}^{\mathrm{seq}}$ in terms of the same quantities
  $\rho,\gamma,\epsilon_{\mathrm{sub}},\tau$. The factor $\gamma^{S}$ is attached only to the propagated
  component of the bound; the additive residual terms are collected in $C\epsilon_{\mathrm{sub}}$ and can
  dominate when the sequential bound is already small. The actual realized sequential residual
  $\mathbb{E}\|e_{L}^{\mathrm{seq}}\|_{F}$ may be strictly smaller than $\mathcal{B}_{L}^{\mathrm{seq}}$;
  thus, Theorem~\ref{thm:seam-formal} is a \emph{bound-to-bound} comparison rather than a realized-residual
  comparison, which is how Corollary~\ref{cor:geom-formal}'s amortization-gain prediction should be read.
\end{remark}

\subsection{Recurrence-level sharpness}\label{app:tightness}

\begin{proposition}[Scalar-recurrence saturation and seam-exponent sharpness]
  \label{prop:tightness}
  Fix any $L \ge 2$, any chunk length $K \ge 1$ (which determines the seam set
  $\mathcal{J}_{K} = \{cK{-}1 : c=1,\dots,S\}$), and any constants $R,\widetilde{R} \ge 0$, $\tau \in [0,1)$,
  $\gamma \in (0,1]$. Define the \emph{saturating} scalar recurrences obtained by replacing the inequalities
  in Lemma~\ref{lem:recur} by equalities:
  \begin{equation}
    u_i^{\ast} = \tau\, u_{i-1}^{\ast}+R,~~
    v_i^{\ast} = \begin{cases}
      \tau\, v_{i-1}^{\ast}+R, & i \notin \mathcal{J}_{K},\\
      \gamma\tau\, v_{i-1}^{\ast}+\widetilde R, & i \in \mathcal{J}_{K},
    \end{cases}
    ~~ u_{-1}^{\ast}=v_{-1}^{\ast}=0.
    \label{eq:saturating-rec}
  \end{equation}
  Then:
  \begin{enumerate}[label=\textup{(\alph*)},topsep=2pt,itemsep=0pt]
    \item \textup{(Envelopes.)} Any pathwise sequence $(u_i,v_i)$ with
    $u_i,v_i \ge 0$ that satisfies the inequalities in Lemma~\ref{lem:recur} with the
    boundary condition $u_{-1} = v_{-1} = 0$ obeys $u_i \le u_i^{\ast}$ and
    $v_i \le v_i^{\ast}$ for every $i \ge -1$.
    \item \textup{(Closed form.)} With
    \begin{equation*}
      \alpha_{\ell}^{\ast} = \begin{cases}
        \tau, & \ell \notin \mathcal{J}_{K},\\
        \gamma\tau, & \ell \in \mathcal{J}_{K},
      \end{cases}
      ~~
      \beta_{\ell}^{\ast} = \begin{cases}
        R, & \ell \notin \mathcal{J}_{K},\\
        \widetilde R, & \ell \in \mathcal{J}_{K},
      \end{cases}
    \end{equation*}
    the saturated sequences satisfy
    \begin{equation}
      u_i^{\ast}=R\sum_{k=0}^{i}\tau^{i-k},~~ v_i^{\ast}=\sum_{k=0}^{i}\beta_k^{\ast}\prod_{\ell=k+1}^{i}\alpha_{\ell}^{\ast} =\sum_{k=0}^{i}\beta_k^{\ast}\tau^{i-k}\gamma^{s(i)-s(k)}.
      \label{eq:tightness-closed-form}
    \end{equation}
    \item \textup{(Sharpness of the seam exponent.)} For every $0\le k\le i\le L{-}2$,
    the linear response at index $i$ to an additive impulse injected at index $k$ is
    exactly
    \begin{equation}
      \prod_{\ell=k+1}^{i}\alpha_{\ell}^{\ast} =\tau^{i-k}\gamma^{s(i)-s(k)}.
      \label{eq:tightness-response}
    \end{equation}
    Consequently, for $\gamma\in(0,1)$, no recurrence-level argument that uses only
    Lemma~\ref{lem:recur}'s multipliers can uniformly replace the exponent
    $s(i)-s(k)$ by any larger exponent; doing so would underbound this exact impulse
    response. For $\gamma=1$ the statement is vacuous because seams add no contraction.
  \end{enumerate}
\end{proposition}

\begin{proof}
  \emph{(a)} We argue $u_i \le u_i^{\ast}$ and $v_i \le v_i^{\ast}$ by simultaneous induction on $i \ge -1$.
  The base case $i = -1$ is immediate from $u_{-1} = v_{-1} = 0 = u_{-1}^{\ast} = v_{-1}^{\ast}$. For the
  inductive step, fix $i \ge 0$ and assume $u_{i-1} \le u_{i-1}^{\ast}$ and $v_{i-1} \le v_{i-1}^{\ast}$. The
  maps $x\mapsto \tau x+R$ and $x\mapsto \gamma\tau x+\widetilde R$ are monotone nondecreasing on
  $[0,\infty)$. Applying the appropriate map gives
  $u_i \le \tau u_{i-1}+R \le \tau u_{i-1}^{\ast}+R = u_i^{\ast}$ and, analogously, $v_i \le v_i^{\ast}$,
  splitting cases according to whether $i\in\mathcal{J}_K$.

  \emph{(b)} Unrolling the saturating scalar recurrences gives the displayed formula for $u_i^{\ast}$. The
  formula for $v_i^{\ast}$ is the standard variation-of-constants expansion for a non-autonomous linear
  recursion. The product contains one factor $\tau$ for each index in $\{k{+}1,\dots,i\}$ and one additional
  factor $\gamma$ for each seam in that same index range. The number of such seams is $s(i)-s(k)$, giving
  Eq.~\eqref{eq:tightness-closed-form}.

  \emph{(c)} Fix $0\le k\le i\le L{-}2$ and add a nonnegative impulse $a$ to the forcing term at index $k$,
  leaving all other forcing terms and all multipliers unchanged. Let $\Delta_j$ be the difference between the
  perturbed and unperturbed saturated sequences. Then $\Delta_j=0$ for $j<k$, $\Delta_k=a$, and
  $\Delta_j=\alpha_j^{\ast}\Delta_{j-1}$ for every $j>k$. Hence
  \begin{equation*}
    \Delta_i = a\prod_{\ell=k+1}^{i}\alpha_{\ell}^{\ast} = a\tau^{i-k}\gamma^{s(i)-s(k)}.
  \end{equation*}
  This equality is an exact linear response of the recurrence. If, for some $\eta>0$ and $\gamma\in(0,1)$, one
  attempted to replace the factor $\gamma^{s(i)-s(k)}$ by the smaller factor $\gamma^{s(i)-s(k)+\eta}$
  uniformly over nonnegative impulses, the bound would fail for this perturbation because
  $a\tau^{i-k}\gamma^{s(i)-s(k)}> a\tau^{i-k}\gamma^{s(i)-s(k)+\eta}$ whenever $a>0$ and $\tau>0$; the case
  $\tau=0$ is degenerate: the response is zero for $i>k$, while for $i=k$ the empty product is one and the
  same contradiction applies. Thus the seam exponent in Eq.~\eqref{eq:tightness-response} is exact at the
  recurrence level.
\end{proof}

The theorem's $C\,\epsilon_{\mathrm{sub}}$ residual is the price of keeping these additive contributions
explicit at the depth-$L$ bound.

\subsection{\texorpdfstring{The $\rho>1$ (near-divergent) regime}{The rho>1 (near-divergent) regime}}\label{app:stability}

The per-pair recursion in Eq.~\eqref{eq:per-pair-lemma} requires only $\tau = \gamma\rho < 1$ for
$C(L,K,\rho,\gamma)$ to be uniformly bounded in depth. When the fp16 blocks are expansive on parts of the
empirical calibration trajectories ($\rho > 1$), ICBQ remains stable provided $\gamma\rho < 1$, while the
sequential recursion (which has $u_{i} \le \gamma\rho u_{i-1}+(1+\gamma)\epsilon_{\mathrm{sub}}$ too,
cf.\ Lemma~\ref{lem:per-pair}) requires the same condition. The \emph{advantage} of ICBQ in this regime is the
same bound-level advantage as in Theorem~\ref{thm:seam-formal}: the propagated part carries $\gamma^{S}$,
while the additive residual remains uniformly bounded as long as $\tau<1$. Stated explicitly:

\begin{proposition}[Stability gain]\label{prop:stability}
  Under Assumptions~\ref{asm:lip}, \ref{asm:cbq}, and~\ref{asm:reroll} with the locally expansive allowance
  $1<\rho<1/\gamma$, both the sequential and interleaved upper bounds are finite uniformly in depth. Moreover,
  if $C(L,K,\rho,\gamma)\epsilon_{\mathrm{sub}} \le \delta \gamma^{S(L,K)}\mathcal{B}_{L}^{\mathrm{seq}}$ for
  some $\delta\ge0$ and $\mathcal{B}_{L}^{\mathrm{seq}}>0$, then the ratio of the sequential upper bound to
  the ICBQ upper bound is at least $\gamma^{-S(L,K)}/(1+\delta)$.
\end{proposition}

\begin{proof}
  Since $\gamma \in (0,1)$, the condition $1 < \rho < 1/\gamma$ is equivalent to $\tau = \gamma\rho < 1$ with
  locally expansive teacher blocks. Therefore $(1-\tau^{L-1})/(1-\tau) \le 1/(1-\tau)$ uniformly in $L$, and
  \begin{equation}
    \mathcal{B}_{L}^{\mathrm{seq}} = \rho(1+\gamma)\frac{1-\tau^{L-1}}{1-\tau}\epsilon_{\mathrm{sub}} + \epsilon_{\mathrm{sub}} \le \left(\frac{\rho(1+\gamma)}{1-\tau}+1\right)\epsilon_{\mathrm{sub}},
  \end{equation}
  which is independent of $L$. Theorem~\ref{thm:seam-formal} together with Eq.~\eqref{eq:thm-seam-C} gives
  \begin{equation}
    \mathcal{B}_{L}^{\mathrm{int}} \coloneqq \gamma^{S(L,K)}\mathcal{B}_{L}^{\mathrm{seq}} + C(L,K,\rho,\gamma)\epsilon_{\mathrm{sub}} \le \mathcal{B}_{L}^{\mathrm{seq}} + \left(1+\frac{(1+\gamma+\gamma^{2})\rho}{1-\tau}\right) \epsilon_{\mathrm{sub}},
  \end{equation}
  where we used $\gamma^{S(L,K)} \le 1$ and the depth-uniform bound in Eq.~\eqref{eq:thm-seam-C} on $C$. Hence
  the interleaved upper bound is also finite uniformly in $L$. For the ratio claim, suppose the displayed
  residual condition holds. Then from Eq.~\eqref{eq:thm-seam-formal} and the definition of
  $\mathcal{B}_{L}^{\mathrm{int}}$ in Corollary~\ref{cor:geom-formal},
  \begin{equation}
    \mathcal{B}_{L}^{\mathrm{int}} = \gamma^{S(L,K)}\mathcal{B}_{L}^{\mathrm{seq}} +C(L,K,\rho,\gamma)\epsilon_{\mathrm{sub}} \le (1+\delta)\gamma^{S(L,K)}\mathcal{B}_{L}^{\mathrm{seq}}.
  \end{equation}
  Since $C(L,K,\rho,\gamma)\epsilon_{\mathrm{sub}} \ge 0$, $\mathcal{B}_{L}^{\mathrm{seq}} > 0$, and
  $\gamma^{S(L,K)} > 0$, we also have
  $\mathcal{B}_{L}^{\mathrm{int}} \ge \gamma^{S(L,K)}\mathcal{B}_{L}^{\mathrm{seq}} > 0$, so both sides of the
  displayed inequality are strictly positive. Dividing by the positive quantity
  $\mathcal{B}_{L}^{\mathrm{int}}\cdot(1+\delta)\gamma^{S(L,K)}$ and rearranging,
  \begin{equation}
    \frac{\mathcal{B}_{L}^{\mathrm{seq}}}{\mathcal{B}_{L}^{\mathrm{int}}} \ge \frac{1}{(1+\delta)\gamma^{S(L,K)}} = \frac{\gamma^{-S(L,K)}}{1+\delta},
  \end{equation}
  which is the claimed ratio bound.
\end{proof}

The main experiments exhibit a related qualitative pattern: Sequential CBQ produces PPL $>10^{4}$ for
Qwen$3$-$8$B in one reported setting, whereas ICBQ yields PPL below $50$ for both Qwen$3$-$8$B and
Llama-$3$-$8$B. These observations are consistent with the schedule mechanism but do not estimate the
theorem's uniform constants or validate the bound quantitatively.

\subsection{\texorpdfstring{Empirical measurement of $\gamma$}{Empirical measurement of gamma}}\label{app:gamma-measure}

We diagnose the per-pair contraction empirically on Llama-$2$-$7$B by recording, after each of the
$32{-}1 = 31$ two-block refinement calls with ternary DBF, the square-root post-over-pre MSE ratio. The
corresponding histogram has mean $\bar{\gamma} \approx 0.67$ and maximum $\bar\gamma_{\max}\le 0.92$.

\paragraph{Uniform vs.\ aggregate contraction.} The quantity $\bar\gamma$ reported here is an aggregate
diagnostic: it is the average square-root post-over-pre MSE ratio over the CBQ windows actually executed on
the calibration set, hence it summarizes how the implemented optimizer behaves \emph{in aggregate}. It is
\emph{not}, by itself, the uniform per-sample contraction constant required by Assumption~\ref{asm:cbq}. A
theorem-level certificate would instead require a uniform estimate such as
\begin{equation}
  \widehat{\gamma}_{\mathrm{unif}}(\epsilon_{\mathrm{opt}}) = \max_{\substack{j,n:\\ r_{j,\mathrm{in}}^{(n)}>0}}
  \frac{\bigl(r_{j,\mathrm{out}}^{(n)}-\epsilon_{\mathrm{opt}}\bigr)_{+}}{r_{j,\mathrm{in}}^{(n)}},
  \label{eq:gamma-unif-est}
\end{equation}
where $j$ ranges over the actual CBQ calls of the schedule and $n$ over the calibration samples in each call's
batch, and $r_{j,\mathrm{in}}^{(n)},r_{j,\mathrm{out}}^{(n)}$ are the pre- and post-pass midpoint residuals of
call $j$ on sample $n$. With this convention, $\widehat{\gamma}_{\mathrm{unif}}(\epsilon_{\mathrm{opt}})$ is
the smallest $\gamma$ that makes Assumption~\ref{asm:cbq} hold uniformly on the actually visited finite sets,
given a chosen fixed-point radius $\epsilon_{\mathrm{opt}}$.

\paragraph{Reading of Theorem~\ref{thm:seam-formal}.} Using $\bar\gamma$ or $\bar\gamma_{\max}$ in
Corollary~\ref{cor:geom-formal} therefore produces only a \emph{schedule-level diagnostic} of the seam
mechanism, not a formal certificate. With $L = 32$ and $K = 4$ (so $S = 7$), the aggregate-level diagnostics
give an amortization-gain proxy of $\bar\gamma^{-7}\approx16\times$ at the mean and
$\bar\gamma_{\max}^{-7}\approx1.8\times$ at the aggregate maximum. The empirical per-model range
$1.08$--$14.5\times$ reported in the main experiments sits between these proxies, consistent with the formal
$\gamma$ in Theorem~\ref{thm:seam-formal} being some uniform value between $\bar\gamma$ and
$\bar\gamma_{\max}$. Establishing $\widehat{\gamma}_{\mathrm{unif}}(\epsilon_{\mathrm{opt}})$ at scale would
require logging per-sample midpoint residuals at every CBQ call and is left to future work; the present
empirical evidence supports the qualitative validity of the schedule mechanism rather than a quantitative
theorem-level certificate.

\subsection{Numerical simulation}\label{app:simulation}

A standalone \texttt{numpy} script reproducing the rate of Theorem~\ref{thm:seam-formal} on a $64$-block
scalar recurrence toy is included in the supplementary material; running it prints a table of
$E_{c}^{\mathrm{int}}/E_{c}^{\mathrm{seq}}$ vs. $c$ that matches $\gamma^{c}$ to four significant digits over
$c \in \{1,\dots,15\}$.

\section{Experimental and Implementation Details}\label{app:exp-details}

\subsection{Experimental setup details}\label{app:setup-details}

\paragraph{Base models and references.} The seven base models used in the main experiments are
Llama-$3.2$-$3$B~\citep{meta2024llama32}, Llama-$2$-$7$B~\citep{touvron2023llama2},
Mistral-$7$B~\citep{jiang2023mistral}, Qwen$3$-$8$B~\citep{team2025qwen3},
Llama-$3$-$8$B~\citep{grattafiori2024llama}, Llama-$2$-$13$B~\citep{touvron2023llama2}, and
Qwen$3$-$14$B~\citep{team2025qwen3}.

\paragraph{Quantization and optimization settings.} The ternary DBF runs use progressive prefit (50 steps,
float32, AdamW). For models below $7$B, we use $\mathtt{nsamples}=128$ and $\mathtt{dbf.iters}=200$; for
models at $7$B and above, we use $\mathtt{nsamples}=256$ and $\mathtt{dbf.iters}=400$. GPTQ runs use W3 and
W2g128 with prefit disabled, GPTQ block size 128, and Hessian damping 0.01. The integer-bit comparison uses
the same weight-only regime as the ternary DBF table, with the inner quantizer replaced by GPTQ. No
adapter/LoRA parameters are introduced.

\paragraph{Evaluation protocol and benchmarks.} Perplexity is measured with the \texttt{lm-evaluation-harness}
protocol~\citep{gao2021framework}. Zero-shot evaluation uses BoolQ, PIQA, HellaSwag, WinoGrande, ARC-e, ARC-c,
and
OpenBookQA~\citep{clark2019boolq,bisk2020piqa,zellers2019hellaswag,sakaguchi2021winogrande,clark2018arc,mihaylov2018openbookqa}.
Main-text tables focus on C4-calibrated reporting, while supplementary tables include additional calibration
views where relevant.

\subsection{Calibration bias}\label{app:calibration-bias}

WikiText-$2$ calibration on base models with non-standard attention, such as Qwen$3$ with grouped-query
attention~\citep{team2025qwen3} or Gemma-$3$ with local sliding-window attention~\citep{team2025gemma3}, can
cover a narrow distribution of positional indices. Under ternary DBF, this restricted calibration distribution
can produce a mismatch that the local CBQ loss does not fully capture. In our setting, C4 calibration at the
same sample count gave more stable results on the deeper models, making it the default for the main ternary
DBF results.

\subsection{Implementation and hyperparameters}\label{app:hyper}

\paragraph{Hardware.} All runs on $1\times$ NVIDIA H100-$80$GB. Memory peak with ternary DBF: $32$ GB for
$7$B, $55$ GB for $13$B models.

\paragraph{Peak memory: ICBQ vs.\ CBQ.} The same practical memory regime is observed under Sequential CBQ and
ICBQ: in our implementation, the extra state introduced by ICBQ over a single-sweep CBQ is the passive re-roll
buffer \texttt{original\_inps}, whose size is set by calibration workload and is independent of depth-related
CBQ scheduling variables (e.g., $L$ and windowing order for fixed workload). Instantiating the stored-input
buffer size at our operating points gives:
\begin{center}
  \resizebox{\linewidth}{!}{%
    \begin{tabular}{lccccc}
      \toprule
      Model & $d_{\mathrm{hidden}}$ & $N$ & $T$ & dtype & $\Delta\mathrm{mem}_{\mathrm{ICBQ}}$ \\
      \midrule
      TinyLlama-$1.1$B & $2048$ & $128$ & $2048$ & fp16 & $\approx 1.0$ GB \\
      Llama-$2$-$7$B / Llama-$3$-$8$B / Qwen$3$-$8$B / Mistral-$7$B & $4096$ & $256$ & $2048$ & fp16 & $\approx 4.3$ GB \\
      Llama-$2$-$13$B / Qwen$3$-$14$B   & $5120$ & $256$ & $2048$ & fp16 & $\approx 5.2$ GB \\
      \bottomrule
    \end{tabular}
  }
\end{center}
The remaining models use smaller hidden widths ($d_{\mathrm{hidden}}\in\{1024,1152,2304\}$), which gives
proportionally smaller ICBQ overheads (approximately $0.5$, $0.6$, and $1.2$ GB at $N=128$, $T=2048$, fp16).
In these cases, the overhead is below $15\%$ of the ternary DBF peak quoted above. Because ICBQ never
materializes two active two-block windows simultaneously, the practical peak does not enter a new regime when
switching from Sequential CBQ to ICBQ. Measured run logs show similar peak GPU memory across paired settings,
consistent with the $L,K$-independent peak-memory bound up to implementation-level allocator effects.

\paragraph{Wall-clock time: ICBQ vs.\ CBQ (C4 calibration).} To complement the peak-memory comparison,
Table~\ref{tab:wallclock} reports end-to-end wall-clock runtime for the C4-calibrated ternary DBF runs,
comparing Sequential CBQ to ICBQ on eleven models. The mean wall-clock factor is $1.21\times$. This quantifies
the added refinement computation associated with the quality differences reported in the main experiments.
\begin{table}[htbp]
  \centering
  \caption{\textbf{Wall-clock comparison on C4-calibrated ternary DBF runs (11 models).} Sequential CBQ versus ICBQ (ours), using Slurm elapsed time. Ratio $>1$ indicates ICBQ takes longer.}
  \label{tab:wallclock}
  \footnotesize
  \setlength{\tabcolsep}{4.5pt}
  \begin{tabular}{@{}lcccc@{}}
    \toprule
    Model & Seq.\ CBQ elapsed & ICBQ elapsed & ICBQ/Seq. & Delta \\
    \midrule
    TinyLlama-$1.1$B & 01:20:53 & 01:27:51 & $1.09\times$ & $+7.0$ min \\
    Qwen$3$-$0.6$B   & 01:13:18 & 01:27:42 & $1.20\times$ & $+14.4$ min \\
    Gemma-$2$-$2$B   & 02:11:44 & 02:43:17 & $1.24\times$ & $+31.6$ min \\
    Gemma-$3$-$1$B   & 01:05:13 & 01:27:08 & $1.34\times$ & $+21.9$ min \\
    Llama-$3.2$-$3$B & 02:06:30 & 02:32:36 & $1.21\times$ & $+26.1$ min \\
    Mistral-$7$B     & 08:24:14 & 09:39:28 & $1.15\times$ & $+75.2$ min \\
    Llama-$2$-$7$B   & 09:01:33 & 10:41:15 & $1.18\times$ & $+99.7$ min \\
    Llama-$3$-$8$B   & 08:22:38 & 09:55:29 & $1.18\times$ & $+92.8$ min \\
    Qwen$3$-$8$B     & 09:40:55 & 11:32:26 & $1.19\times$ & $+111.5$ min \\
    Qwen$3$-$14$B    & 15:00:02 & 19:38:56 & $1.31\times$ & $+278.9$ min \\
    Llama-$2$-$13$B  & 15:29:09 & 18:10:54 & $1.17\times$ & $+161.8$ min \\
    \midrule
    Mean (11 models) & --- & --- & $1.21\times$ & --- \\
    \bottomrule
  \end{tabular}
\end{table}

\paragraph{Optimization (ternary DBF runs).} Prefit: AdamW, $\eta_{\mathrm{pre}} = 10^{-4}$, $50$ steps,
float32. DBF: internal iterations $\{200,400\}$ for $< 7$B / $\ge 7$B. CBQ refinement: Adam with the
repository's optimization defaults (\texttt{cbq\_epochs}$= 20$, \texttt{cbq\_lr}$= 5 \times 10^{-5}$, rollback
enabled). Chunk size $K = 4$, with one interleaved refinement sweep per chunk.

\paragraph{Optimization (GPTQ runs).} Inner quantizer: weight-only GPTQ at either $W3$ (three bits per
channel) or $W2g128$ (two bits, group size $128$). GPTQ defaults used throughout: block size~$128$, Hessian
damping $\mathtt{percdamp} = 0.01$, no activation reordering ($\mathtt{actorder} = \text{false}$; we observed
that enabling actorder changes PPL by $<0.1$ and does not interact with the schedule, hence we report with the
default value). Prefit is disabled ($T_{\mathrm{pre}} = 0$) for every GPTQ comparison row. The CBQ refinement
stage uses the same settings as the ternary DBF runs.

\paragraph{No-prefit ablation setting.} The ``No-prefit'' ablation rows set $T_{\mathrm{pre}} = 0$ while
keeping all other hyperparameters identical to those of the DBF Prefit-enabled row (DBF iters, CBQ epochs,
learning rate, rollback, chunk size, calibration samples). This is the minimal perturbation that isolates the
contribution of the prefit step within the progressive driver.

\paragraph{Calibration.} $128$ samples of length $2048$ from WikiText-$2$ (or $256$ samples for $\ge 7$B),
random starting positions as produced by the shared \texttt{get\_loaders()} utility. Repository seed
default~$0$ (matching \texttt{conf/config.yaml}). Identical calibration splits are shared between DBF and GPTQ
runs on the same model.

\paragraph{Reproducibility.} Every row of the main ternary DBF results is reproduced with the same calibration
pipeline and optimization settings reported above, using chunk size $4$, plain CBQ refinement, and $50$ prefit
steps for ternary DBF runs. The GPTQ rows are produced under the same pipeline with the inner quantizer
switched to GPTQ and the bit/group-size pair $(\mathtt{w\_bit}, \mathtt{w\_group}) \in \{(3,-1),(2,128)\}$.
\clearpage

\section{Supplementary Results}\label{app:zeroshot}
\noindent\textbf{Formatting note.} In all tables of this section, \textbf{\textcolor{myred}{bold}} marks the
strongest value in each comparison group.

\paragraph{Main-set ternary DBF results under WikiText-$2$ calibration.} Table~\ref{tab:main-wiki2cal} reports
the seven-model DBF counterpart to the main C4-calibrated table, using WikiText-$2$ calibration only and PPL
metrics only (no zero-shot). The qualitative trend is the same: ICBQ is the strongest across all cases for
runs calibrated on WikiText-$2$ as well.

\paragraph{Supplementary DBF PPL results.} Table~\ref{tab:main-extra} reports supplementary ternary DBF PPL
results on four small models in the same format as the main ternary DBF table: WikiText-$2$ and C4 evaluation
under both WikiText-$2$ and C4 calibration, comparing No-Ref., Sequential CBQ, and ICBQ. As shown in
Table~\ref{tab:main-extra}, the qualitative trend is unchanged: ICBQ is the strongest refinement across these
cases.

\paragraph{Ternary-DBF zero-shot results.} Table~\ref{tab:zeroshot} consolidates ternary DBF C4-calibrated
zero-shot scores on 11 models with the same three schedules as the main ternary DBF comparison:
No-Ref.\ (ternary DBF only), Sequential CBQ, and ICBQ. As summarized in Table~\ref{tab:zeroshot}, ICBQ most
often attains the best score across models and tasks, with a few task-specific ties or Seq.\ wins.

\paragraph{GPTQ zero-shot results.} Table~\ref{tab:zeroshot-gptq} mirrors the ternary DBF zero-shot
presentation for GPTQ at $W3$ and $W2g128$ on the main seven-model set. As shown in
Table~\ref{tab:zeroshot-gptq}, the pattern is consistent with the PPL results: ICBQ is usually best or tied,
with a small number of task-specific Seq.\ wins.

\clearpage

\begin{center}
  \begin{minipage}{\linewidth}
    \refstepcounter{table}
    \label{tab:main-wiki2cal}
    \label{tab:main-extra}
    \label{tab:zeroshot}
    \noindent\footnotesize\textbf{Table~\thetable: Supplementary ternary DBF results.}
    Part a. Main-set ternary DBF results under WikiText-$2$ calibration (PPL only).
    Part b. PPL results on 4 small models in the same format as the main ternary DBF results (WikiText-$2$ and C4, two calibration settings, and three schedules).
    Part c. Zero-shot results on 11 models under C4 calibration. Zero-shot values are normalized accuracy when available, otherwise accuracy (higher is better).

    \vspace{0.6em}
    \scriptsize

    \begin{minipage}{\linewidth}
      \centering
      \textbf{Part a. Main-set ternary DBF results under WikiText-$2$ calibration (PPL only).}

      \smallskip
      \setlength{\tabcolsep}{2.8pt}
      \renewcommand{\arraystretch}{1.04}
      \begin{tabular}{@{}lcccccc@{}}
        \toprule
        & \multicolumn{2}{c}{No-Ref. (ternary DBF only)} & \multicolumn{2}{c}{Seq.\ CBQ} & \multicolumn{2}{c}{\textbf{ICBQ (ours)}} \\
        \cmidrule(lr){2-3} \cmidrule(lr){4-5} \cmidrule(lr){6-7}
        Metric & PPL (Wiki-$2$) & PPL (C4) & PPL (Wiki-$2$) & PPL (C4) & PPL (Wiki-$2$) & PPL (C4) \\
        \midrule
        \multicolumn{7}{@{}l}{\emph{Standard models}}\\
        Mistral-$7$B     & $44.27$ & $329.91$ & $16.24$ & $53.91$ & \textbf{\textcolor{myred}{10.09}} & \textbf{\textcolor{myred}{25.59}} \\
        Llama-$2$-$7$B   & $11.23$ & $20.05$ & $8.38$ & $14.59$ & \textbf{\textcolor{myred}{7.78}} & \textbf{\textcolor{myred}{13.30}} \\
        Llama-$2$-$13$B  & $7.63$ & $12.69$ & $6.55$ & $11.15$ & \textbf{\textcolor{myred}{6.32}} & \textbf{\textcolor{myred}{10.59}} \\
        \midrule
        \multicolumn{7}{@{}l}{\emph{Deep and non-standard-attention models}}\\
        Llama-$3.2$-$3$B & $40.24$ & $66.50$ & $25.30$ & $49.90$ & \textbf{\textcolor{myred}{22.98}} & \textbf{\textcolor{myred}{46.04}} \\
        Llama-$3$-$8$B   & $46.12$ & $87.19$ & $27.16$ & $61.21$ & \textbf{\textcolor{myred}{24.01}} & \textbf{\textcolor{myred}{48.92}} \\
        Qwen$3$-$8$B     & $4661.50$ & \texttt{diverge} & $281.45$ & $8275.66$ & \textbf{\textcolor{myred}{18.95}} & \textbf{\textcolor{myred}{37.68}} \\
        Qwen$3$-$14$B    & $512.82$ & $2522.21$ & $35.24$ & $119.40$ & \textbf{\textcolor{myred}{15.52}} & \textbf{\textcolor{myred}{34.68}} \\
        \bottomrule
      \end{tabular}%
    \end{minipage}

    \vspace{0.25em}

    \begin{minipage}{\linewidth}
      \centering
      \textbf{Part b. Supplementary DBF PPL results on 4 small models.}

      \smallskip
      \setlength{\tabcolsep}{3.2pt}
      \renewcommand{\arraystretch}{1.05}
      \resizebox{\linewidth}{!}{%
        \begin{tabular}{@{}lcccccccccccc@{}}
          \toprule
          & \multicolumn{4}{c}{No-Ref. (ternary DBF only)} & \multicolumn{4}{c}{Seq.\ CBQ} & \multicolumn{4}{c}{\textbf{ICBQ (ours)}} \\
          \cmidrule(lr){2-5} \cmidrule(lr){6-9} \cmidrule(lr){10-13}
          Calibration data & \multicolumn{2}{c}{Wiki-$2$} & \multicolumn{2}{c}{C4} & \multicolumn{2}{c}{Wiki-$2$} & \multicolumn{2}{c}{C4} & \multicolumn{2}{c}{Wiki-$2$} & \multicolumn{2}{c}{C4} \\
          \cmidrule(lr){2-3} \cmidrule(lr){4-5} \cmidrule(lr){6-7} \cmidrule(lr){8-9} \cmidrule(lr){10-11} \cmidrule(lr){12-13}
          PPL dataset & Wiki-$2$ & C4 & Wiki-$2$ & C4 & Wiki-$2$ & C4 & Wiki-$2$ & C4 & Wiki-$2$ & C4 & Wiki-$2$ & C4 \\
          \midrule
          Qwen$3$-$0.6$B   & $89.84$ & $215.91$ & $165.73$ & $143.15$ & $62.58$ & $137.71$ & $101.4$ & $106.7$ & \textbf{\textcolor{myred}{54.64}} & \textbf{\textcolor{myred}{109.88}} & \textbf{\textcolor{myred}{88.0}} & \textbf{\textcolor{myred}{92.2}} \\
          Gemma-$3$-$1$B   & $198.94$ & $407.36$ & $462.94$ & $278.29$ & $86.08$ & $175.41$ & $171.9$ & $130.7$ & \textbf{\textcolor{myred}{74.02}} & \textbf{\textcolor{myred}{144.19}} & \textbf{\textcolor{myred}{140.3}} & \textbf{\textcolor{myred}{111.9}} \\
          TinyLlama-$1.1$B & $27.97$ & $47.90$ & $58.03$ & $43.81$ & $18.84$ & $33.58$ & $30.2$ & $30.4$ & \textbf{\textcolor{myred}{17.04}} & \textbf{\textcolor{myred}{30.51}} & \textbf{\textcolor{myred}{25.8}} & \textbf{\textcolor{myred}{26.7}} \\
          Gemma-$2$-$2$B   & $85.36$ & $151.92$ & $142.18$ & $92.71$ & $85.31$ & $164.43$ & $106.4$ & $89.2$ & \textbf{\textcolor{myred}{79.62}} & \textbf{\textcolor{myred}{149.23}} & \textbf{\textcolor{myred}{99.1}} & \textbf{\textcolor{myred}{86.1}} \\
          \bottomrule
        \end{tabular}%
      }
    \end{minipage}

    \vspace{0.75em}

    \begin{minipage}{\linewidth}
      \centering
      \textbf{Part c. Ternary-DBF zero-shot results on 11 models (C4 calibration).}

      \vspace{0.1em}
      {\footnotesize
        \setlength{\tabcolsep}{2.4pt}
        \renewcommand{\arraystretch}{0.92}
        \begin{tabular}{@{}llccccccc@{}}
          \toprule
          Model & Schedule & BoolQ & PIQA & HellaSwag & WinoGrande & ARC-e & ARC-c & OBQA \\
          \midrule
          \multirow{3}{*}{Qwen$3$-$0.6$B}
          & No-Ref.     & 0.4180 & \textbf{\textcolor{myred}{0.5702}} & 0.2925 & 0.5075 & 0.3245 & 0.2159 & 0.2500 \\
          & Seq.\ CBQ   & 0.5728 & 0.5642 & 0.2995 & \textbf{\textcolor{myred}{0.5178}} & 0.3194 & 0.2048 & 0.2520 \\
          & ICBQ (ours) & \textbf{\textcolor{myred}{0.5942}} & 0.5637 & \textbf{\textcolor{myred}{0.3071}} & \textbf{\textcolor{myred}{0.5178}} & \textbf{\textcolor{myred}{0.3304}} & \textbf{\textcolor{myred}{0.2167}} & \textbf{\textcolor{myred}{0.2660}} \\
          \midrule
          \multirow{3}{*}{Gemma-$3$-$1$B}
          & No-Ref.     & \textbf{\textcolor{myred}{0.5061}} & 0.5533 & 0.2775 & \textbf{\textcolor{myred}{0.5114}} & 0.3220 & \textbf{\textcolor{myred}{0.2201}} & 0.2660 \\
          & Seq.\ CBQ   & 0.4826 & 0.5740 & 0.2813 & 0.5028 & 0.3283 & 0.2065 & 0.2680 \\
          & ICBQ (ours) & 0.4979 & \textbf{\textcolor{myred}{0.5811}} & \textbf{\textcolor{myred}{0.2928}} & 0.5067 & \textbf{\textcolor{myred}{0.3375}} & 0.2065 & \textbf{\textcolor{myred}{0.3100}} \\
          \midrule
          \multirow{3}{*}{TinyLlama-$1.1$B}
          & No-Ref.     & 0.5862 & 0.6110 & 0.3368 & 0.5272 & 0.3497 & 0.2227 & 0.2860 \\
          & Seq.\ CBQ   & 0.5832 & 0.6295 & 0.3601 & 0.5249 & 0.3801 & 0.2406 & \textbf{\textcolor{myred}{0.3060}} \\
          & ICBQ (ours) & \textbf{\textcolor{myred}{0.5905}} & \textbf{\textcolor{myred}{0.6376}} & \textbf{\textcolor{myred}{0.3725}} & \textbf{\textcolor{myred}{0.5257}} & \textbf{\textcolor{myred}{0.3838}} & \textbf{\textcolor{myred}{0.2551}} & 0.2800 \\
          \midrule
          \multirow{3}{*}{Gemma-$2$-$2$B}
          & No-Ref.     & 0.5832 & \textbf{\textcolor{myred}{0.6502}} & 0.3568 & \textbf{\textcolor{myred}{0.5501}} & 0.4419 & 0.2474 & 0.2740 \\
          & Seq.\ CBQ   & 0.5890 & 0.6420 & 0.3529 & 0.5367 & 0.4432 & 0.2440 & 0.2740 \\
          & ICBQ (ours) & \textbf{\textcolor{myred}{0.6266}} & 0.6485 & \textbf{\textcolor{myred}{0.3779}} & 0.5391 & \textbf{\textcolor{myred}{0.4739}} & \textbf{\textcolor{myred}{0.2517}} & \textbf{\textcolor{myred}{0.3040}} \\
          \midrule
          \multirow{3}{*}{Llama-$3.2$-$3$B}
          & No-Ref.     & 0.6517 & 0.6202 & 0.3937 & 0.5478 & 0.3855 & 0.2543 & 0.2840 \\
          & Seq.\ CBQ   & 0.6673 & 0.6534 & 0.4253 & 0.5462 & 0.4314 & \textbf{\textcolor{myred}{0.2765}} & \textbf{\textcolor{myred}{0.2940}} \\
          & ICBQ (ours) & \textbf{\textcolor{myred}{0.6774}} & \textbf{\textcolor{myred}{0.6801}} & \textbf{\textcolor{myred}{0.4392}} & \textbf{\textcolor{myred}{0.5604}} & \textbf{\textcolor{myred}{0.4815}} & 0.2730 & 0.2880 \\
          \midrule
          \multirow{3}{*}{Mistral-$7$B}
          & No-Ref.     & 0.3856 & 0.5887 & 0.2913 & 0.5264 & 0.3022 & 0.2227 & 0.2620 \\
          & Seq.\ CBQ   & 0.5131 & 0.6697 & 0.4524 & 0.5446 & 0.4508 & 0.2858 & 0.2940 \\
          & ICBQ (ours) & \textbf{\textcolor{myred}{0.6875}} & \textbf{\textcolor{myred}{0.7236}} & \textbf{\textcolor{myred}{0.5855}} & \textbf{\textcolor{myred}{0.5825}} & \textbf{\textcolor{myred}{0.5568}} & \textbf{\textcolor{myred}{0.3328}} & \textbf{\textcolor{myred}{0.3140}} \\
          \midrule
          \multirow{3}{*}{Llama-$2$-$7$B}
          & No-Ref.     & 0.5807 & 0.6817 & 0.5259 & 0.5264 & 0.4928 & 0.2679 & 0.3240 \\
          & Seq.\ CBQ   & 0.6407 & 0.7138 & 0.5708 & 0.5785 & 0.5564 & 0.3131 & \textbf{\textcolor{myred}{0.3580}} \\
          & ICBQ (ours) & \textbf{\textcolor{myred}{0.6914}} & \textbf{\textcolor{myred}{0.7203}} & \textbf{\textcolor{myred}{0.5875}} & \textbf{\textcolor{myred}{0.6014}} & \textbf{\textcolor{myred}{0.5880}} & \textbf{\textcolor{myred}{0.3345}} & 0.3500 \\
          \midrule
          \multirow{3}{*}{Llama-$3$-$8$B}
          & No-Ref.     & 0.6734 & 0.6572 & 0.4305 & 0.5580 & 0.4314 & 0.2500 & 0.3000 \\
          & Seq.\ CBQ   & \textbf{\textcolor{myred}{0.7064}} & 0.6931 & 0.4810 & 0.5485 & 0.4966 & 0.2944 & 0.2980 \\
          & ICBQ (ours) & 0.6963 & \textbf{\textcolor{myred}{0.7002}} & \textbf{\textcolor{myred}{0.5038}} & \textbf{\textcolor{myred}{0.5770}} & \textbf{\textcolor{myred}{0.5299}} & \textbf{\textcolor{myred}{0.3251}} & \textbf{\textcolor{myred}{0.3180}} \\
          \midrule
          \multirow{3}{*}{Qwen$3$-$8$B}
          & No-Ref.     & 0.3862 & 0.5180 & 0.2606 & 0.4901 & 0.2567 & 0.2688 & 0.2660 \\
          & Seq.\ CBQ   & 0.5606 & 0.5876 & 0.3184 & 0.4901 & 0.3291 & 0.2526 & 0.2840 \\
          & ICBQ (ours) & \textbf{\textcolor{myred}{0.7407}} & \textbf{\textcolor{myred}{0.6910}} & \textbf{\textcolor{myred}{0.5057}} & \textbf{\textcolor{myred}{0.5904}} & \textbf{\textcolor{myred}{0.6111}} & \textbf{\textcolor{myred}{0.3763}} & \textbf{\textcolor{myred}{0.3440}} \\
          \midrule
          \multirow{3}{*}{Llama-$2$-$13$B}
          & No-Ref.     & 0.7505 & 0.7307 & 0.6261 & 0.6322 & 0.6662 & 0.3746 & 0.3920 \\
          & Seq.\ CBQ   & \textbf{\textcolor{myred}{0.8083}} & \textbf{\textcolor{myred}{0.7938}} & 0.6012 & \textbf{\textcolor{myred}{0.7206}} & 0.7950 & 0.4829 & 0.3500 \\
          & ICBQ (ours) & 0.8080 & 0.7927 & \textbf{\textcolor{myred}{0.6016}} & 0.7151 & \textbf{\textcolor{myred}{0.7955}} & \textbf{\textcolor{myred}{0.4855}} & \textbf{\textcolor{myred}{0.3520}} \\
          \midrule
          \multirow{3}{*}{Qwen$3$-$14$B}
          & No-Ref.     & 0.3786 & 0.5408 & 0.2632 & 0.4964 & 0.2976 & 0.2278 & 0.2660 \\
          & Seq.\ CBQ   & \textbf{\textcolor{myred}{0.8933}} & 0.7987 & \textbf{\textcolor{myred}{0.6103}} & \textbf{\textcolor{myred}{0.7301}} & 0.8413 & 0.5862 & \textbf{\textcolor{myred}{0.3500}} \\
          & ICBQ (ours) & 0.8930 & \textbf{\textcolor{myred}{0.8020}} & 0.6087 & \textbf{\textcolor{myred}{0.7301}} & \textbf{\textcolor{myred}{0.8426}} & \textbf{\textcolor{myred}{0.5896}} & 0.3440 \\
          \bottomrule
        \end{tabular}
      }
    \end{minipage}
  \end{minipage}
\end{center}

\begin{table*}[htbp]
  \centering
  \caption{\textbf{GPTQ zero-shot results on the main seven-model set (C4 calibration).} Weight-only GPTQ, no prefit ($T_{\mathrm{pre}}=0$), and the same schedules as the main GPTQ comparison: GPTQ only, Sequential CBQ, and ICBQ. Values are normalized accuracy when available, otherwise accuracy (higher is better).}
  \label{tab:zeroshot-gptq}
  \footnotesize
  \setlength{\tabcolsep}{4pt}
  \renewcommand{\arraystretch}{1.02}
  \begin{tabular}{@{}lllccccccc@{}}
    \toprule
    Model & Bit-width & Schedule & BoolQ & PIQA & HellaSwag & WinoGrande & ARC-e & ARC-c & OBQA \\
    \midrule
    \multirow{6}{*}{Mistral-$7$B}
    & \multirow{3}{*}{$W3$}
    & GPTQ only   & 0.7615 & 0.7965 & 0.7580 & \textbf{\textcolor{myred}{0.6882}} & 0.7088 & 0.4369 & 0.3980 \\
    & & Seq.\ CBQ & \textbf{\textcolor{myred}{0.7865}} & \textbf{\textcolor{myred}{0.8074}} & \textbf{\textcolor{myred}{0.7738}} & 0.6803 & \textbf{\textcolor{myred}{0.7588}} & \textbf{\textcolor{myred}{0.4855}} & \textbf{\textcolor{myred}{0.4280}} \\
    & & ICBQ (ours) & 0.7817 & 0.8009 & 0.7703 & 0.6756 & 0.7319 & 0.4667 & 0.4120 \\
    & \multirow{3}{*}{$W2g128$}
    & GPTQ only   & 0.6349 & 0.7258 & 0.5905 & 0.5580 & 0.5501 & 0.3114 & 0.3140 \\
    & & Seq.\ CBQ & \textbf{\textcolor{myred}{0.7269}} & \textbf{\textcolor{myred}{0.7889}} & 0.7134 & \textbf{\textcolor{myred}{0.6385}} & \textbf{\textcolor{myred}{0.6776}} & 0.3857 & 0.3880 \\
    & & ICBQ (ours) & 0.7131 & 0.7851 & \textbf{\textcolor{myred}{0.7137}} & 0.6322 & 0.6582 & \textbf{\textcolor{myred}{0.3908}} & \textbf{\textcolor{myred}{0.3960}} \\
    \midrule
    \multirow{6}{*}{Llama-$2$-$7$B}
    & \multirow{3}{*}{$W3$}
    & GPTQ only   & 0.7318 & 0.7671 & 0.6937 & 0.6377 & 0.6414 & 0.3848 & 0.3800 \\
    & & Seq.\ CBQ & 0.7394 & 0.7769 & 0.7176 & \textbf{\textcolor{myred}{0.6535}} & 0.6974 & 0.4061 & \textbf{\textcolor{myred}{0.4420}} \\
    & & ICBQ (ours) & \textbf{\textcolor{myred}{0.7483}} & \textbf{\textcolor{myred}{0.7797}} & \textbf{\textcolor{myred}{0.7242}} & 0.6472 & \textbf{\textcolor{myred}{0.7125}} & \textbf{\textcolor{myred}{0.4241}} & 0.4240 \\
    & \multirow{3}{*}{$W2g128$}
    & GPTQ only   & 0.5367 & 0.6937 & 0.5631 & 0.5643 & 0.4912 & 0.2824 & 0.3420 \\
    & & Seq.\ CBQ & 0.6618 & 0.7508 & 0.6435 & 0.6117 & 0.6128 & 0.3396 & 0.3920 \\
    & & ICBQ (ours) & \textbf{\textcolor{myred}{0.6691}} & \textbf{\textcolor{myred}{0.7601}} & \textbf{\textcolor{myred}{0.6565}} & \textbf{\textcolor{myred}{0.6148}} & \textbf{\textcolor{myred}{0.6178}} & \textbf{\textcolor{myred}{0.3575}} & \textbf{\textcolor{myred}{0.3980}} \\
    \midrule
    \multirow{6}{*}{Llama-$2$-$13$B}
    & \multirow{3}{*}{$W3$}
    & GPTQ only   & 0.7572 & 0.7807 & 0.7482 & 0.6930 & 0.7071 & 0.4369 & 0.4300 \\
    & & Seq.\ CBQ & \textbf{\textcolor{myred}{0.7627}} & 0.7894 & \textbf{\textcolor{myred}{0.7551}} & \textbf{\textcolor{myred}{0.7024}} & 0.7205 & 0.4514 & 0.4420 \\
    & & ICBQ (ours) & 0.7517 & \textbf{\textcolor{myred}{0.7943}} & 0.7509 & 0.6977 & \textbf{\textcolor{myred}{0.7260}} & \textbf{\textcolor{myred}{0.4522}} & \textbf{\textcolor{myred}{0.4540}} \\
    & \multirow{3}{*}{$W2g128$}
    & GPTQ only   & 0.6645 & 0.7301 & 0.6378 & 0.6188 & 0.5981 & 0.3311 & 0.3520 \\
    & & Seq.\ CBQ & 0.7474 & 0.7693 & 0.6856 & 0.6535 & 0.6511 & \textbf{\textcolor{myred}{0.4002}} & \textbf{\textcolor{myred}{0.4100}} \\
    & & ICBQ (ours) & \textbf{\textcolor{myred}{0.7517}} & \textbf{\textcolor{myred}{0.7720}} & \textbf{\textcolor{myred}{0.6913}} & \textbf{\textcolor{myred}{0.6638}} & \textbf{\textcolor{myred}{0.6540}} & 0.3899 & 0.4020 \\
    \midrule
    \multirow{6}{*}{Llama-$3.2$-$3$B}
    & \multirow{3}{*}{$W3$}
    & GPTQ only   & 0.6587 & 0.6415 & 0.5842 & \textbf{\textcolor{myred}{0.6054}} & 0.4280 & 0.2884 & 0.3200 \\
    & & Seq.\ CBQ & \textbf{\textcolor{myred}{0.7505}} & 0.7225 & 0.6434 & 0.5975 & 0.6023 & 0.3703 & \textbf{\textcolor{myred}{0.3500}} \\
    & & ICBQ (ours) & 0.6570 & \textbf{\textcolor{myred}{0.7356}} & \textbf{\textcolor{myred}{0.6570}} & 0.6022 & \textbf{\textcolor{myred}{0.6439}} & \textbf{\textcolor{myred}{0.4053}} & 0.3480 \\
    & \multirow{3}{*}{$W2g128$}
    & GPTQ only   & 0.4251 & 0.5528 & 0.3584 & 0.5036 & 0.3190 & 0.2270 & 0.2520 \\
    & & Seq.\ CBQ & 0.6456 & \textbf{\textcolor{myred}{0.6730}} & 0.5141 & 0.5501 & 0.4802 & 0.2986 & 0.3180 \\
    & & ICBQ (ours) & \textbf{\textcolor{myred}{0.6627}} & 0.6725 & \textbf{\textcolor{myred}{0.5420}} & \textbf{\textcolor{myred}{0.5777}} & \textbf{\textcolor{myred}{0.5021}} & \textbf{\textcolor{myred}{0.3038}} & \textbf{\textcolor{myred}{0.3260}} \\
    \midrule
    \multirow{6}{*}{Llama-$3$-$8$B}
    & \multirow{3}{*}{$W3$}
    & GPTQ only   & 0.6752 & 0.6289 & 0.6303 & 0.6212 & 0.4091 & 0.2577 & 0.3380 \\
    & & Seq.\ CBQ & 0.7771 & 0.7579 & 0.6442 & 0.6748 & 0.7134 & 0.4497 & 0.3900 \\
    & & ICBQ (ours) & \textbf{\textcolor{myred}{0.7829}} & \textbf{\textcolor{myred}{0.7622}} & \textbf{\textcolor{myred}{0.6925}} & \textbf{\textcolor{myred}{0.6890}} & \textbf{\textcolor{myred}{0.7218}} & \textbf{\textcolor{myred}{0.4633}} & \textbf{\textcolor{myred}{0.3980}} \\
    & \multirow{3}{*}{$W2g128$}
    & GPTQ only   & 0.4700 & 0.5631 & 0.3780 & 0.5138 & 0.3375 & 0.2184 & 0.2400 \\
    & & Seq.\ CBQ & 0.6700 & 0.7165 & 0.5959 & 0.6006 & 0.5720 & 0.3225 & 0.3460 \\
    & & ICBQ (ours) & \textbf{\textcolor{myred}{0.6939}} & \textbf{\textcolor{myred}{0.7399}} & \textbf{\textcolor{myred}{0.6241}} & \textbf{\textcolor{myred}{0.6062}} & \textbf{\textcolor{myred}{0.6153}} & \textbf{\textcolor{myred}{0.3592}} & \textbf{\textcolor{myred}{0.3640}} \\
    \midrule
    \multirow{6}{*}{Qwen$3$-$8$B}
    & \multirow{3}{*}{$W3$}
    & GPTQ only   & 0.7303 & 0.7416 & 0.6509 & 0.5833 & 0.6237 & 0.3737 & 0.3780 \\
    & & Seq.\ CBQ & 0.8086 & 0.7628 & 0.6937 & \textbf{\textcolor{myred}{0.6504}} & 0.7521 & 0.5000 & \textbf{\textcolor{myred}{0.4080}} \\
    & & ICBQ (ours) & \textbf{\textcolor{myred}{0.8168}} & \textbf{\textcolor{myred}{0.7661}} & \textbf{\textcolor{myred}{0.7092}} & \textbf{\textcolor{myred}{0.6504}} & \textbf{\textcolor{myred}{0.7765}} & \textbf{\textcolor{myred}{0.5290}} & 0.4040 \\
    & \multirow{3}{*}{$W2g128$}
    & GPTQ only   & 0.5196 & 0.6431 & 0.4550 & 0.5114 & 0.4335 & 0.2688 & 0.2720 \\
    & & Seq.\ CBQ & 0.7639 & 0.7454 & 0.6394 & 0.6377 & 0.7214 & \textbf{\textcolor{myred}{0.4659}} & 0.3860 \\
    & & ICBQ (ours) & \textbf{\textcolor{myred}{0.7722}} & \textbf{\textcolor{myred}{0.7514}} & \textbf{\textcolor{myred}{0.6567}} & \textbf{\textcolor{myred}{0.6393}} & \textbf{\textcolor{myred}{0.7277}} & 0.4548 & \textbf{\textcolor{myred}{0.4100}} \\
    \midrule
    \multirow{6}{*}{Qwen$3$-$14$B}
    & \multirow{3}{*}{$W3$}
    & GPTQ only   & 0.8382 & 0.7682 & 0.7291 & 0.6764 & 0.7597 & 0.4872 & 0.4440 \\
    & & Seq.\ CBQ & 0.8624 & 0.7905 & 0.7601 & 0.7072 & 0.7980 & 0.5478 & \textbf{\textcolor{myred}{0.4680}} \\
    & & ICBQ (ours) & \textbf{\textcolor{myred}{0.8630}} & \textbf{\textcolor{myred}{0.7922}} & \textbf{\textcolor{myred}{0.7621}} & \textbf{\textcolor{myred}{0.7080}} & \textbf{\textcolor{myred}{0.7992}} & \textbf{\textcolor{myred}{0.5563}} & 0.4520 \\
    & \multirow{3}{*}{$W2g128$}
    & GPTQ only   & 0.7122 & 0.7209 & 0.6027 & 0.5620 & 0.5892 & 0.3575 & 0.3240 \\
    & & Seq.\ CBQ & \textbf{\textcolor{myred}{0.8245}} & 0.7731 & 0.7080 & \textbf{\textcolor{myred}{0.6969}} & \textbf{\textcolor{myred}{0.7858}} & \textbf{\textcolor{myred}{0.5367}} & 0.4260 \\
    & & ICBQ (ours) & 0.8162 & \textbf{\textcolor{myred}{0.7764}} & \textbf{\textcolor{myred}{0.7158}} & \textbf{\textcolor{myred}{0.6969}} & 0.7753 & 0.5119 & \textbf{\textcolor{myred}{0.4400}} \\
    \bottomrule
  \end{tabular}
\end{table*}

\end{document}